\documentclass[sn-mathphys,Numbered]{sn-jnl}% Math and Physical Sciences Reference Style
\usepackage[cp1251]{inputenc}
\usepackage{graphicx}

\usepackage{amsmath, amsxtra, amssymb, amsfonts, amstext}
\usepackage{mathtools}

\usepackage{fullpage}
\usepackage{nicefrac}
\usepackage{xspace}
\usepackage[algo2e,ruled,vlined,linesnumbered]{algorithm2e}
\usepackage{algpseudocode}
\usepackage{multirow}

\usepackage{footnote}
\usepackage{bm}
\usepackage{bbm}
\usepackage{enumerate}
\usepackage{amsthm}% % necessary to comment this out?
\usepackage{thmtools}
\usepackage{thm-restate}
\usepackage{lineno}
\usepackage{pdflscape}
\usepackage{afterpage}
\usepackage[final]{pdfpages}
\usepackage{algorithmicx}
\usepackage{algorithm}
\usepackage{algpseudocode}
\usepackage{mathrsfs}%
\usepackage[title]{appendix}%
\usepackage{xcolor}%
\usepackage{textcomp}%
\usepackage{manyfoot}%
\usepackage{booktabs}%
\usepackage{listings}%
\newtheorem{theorem}{Theorem}%  meant for continuous numbers
\newtheorem{proposition}[theorem]{Proposition}%
\newtheorem{corollary}[theorem]{Corollary}%
\newtheorem{lemma}[theorem]{Lemma}%
\newtheorem{claim}[theorem]{Claim}%

\theoremstyle{remark}%

\theoremstyle{definition}%
\newtheorem{definition}{Definition}%

\newcommand{\N}{\mathbb{N}}
\newcommand{\R}{\mathbb{R}}

\newcommand{\A}{\mathcal{A}}

\newcommand{\calD}{\mathcal{D}}

\renewcommand{\epsilon}{\varepsilon}
\newcommand{\eps}{\varepsilon}

\DeclareMathOperator{\E}{\mathbb{E}}

\DeclareMathOperator{\mut}{mut}

\newcommand{\mutD}{\mut_{\calD}}

\newcommand{\ooea}{$(1 + 1)$-\text{EA}\xspace}
\newcommand{\ooeaD}{$(1 + 1)$-\text{EA}$_{\calD}$\xspace}
\newcommand{\olea}{$(1 + \lambda)$-EA\xspace}
\newcommand{\moea}{$(\mu + 1)$-EA\xspace}
\newcommand{\moga}{$(\mu + 1)$-GA\xspace}

\newcommand{\OneMax}{\textsc{OneMax}\xspace}
\newcommand{\onemax}{\textsc{OneMax}\xspace}
\newcommand{\OM}{\textsc{Om}\xspace}

\newcommand{\jump}{\textsc{Jump}\xspace}
\newcommand{\hottopic}{\textsc{HotTopic}\xspace}

\newcommand{\leadingones}{\textsc{LeadingOnes}\xspace}

\begin{document}

\title[Tight Runtime Bounds for Static Unary Unbiased EAs on Linear Functions]{Tight Runtime Bounds for Static Unary Unbiased\\ Evolutionary Algorithms on Linear Functions}


\author*[1]{\fnm{Carola} \sur{Doerr}}\email{carola.doerr@lip6.fr}
\equalcont{All authors contributed equally to this work.}

\author[1,2,3]{\fnm{Duri Andrea} \sur{Janett}}
\equalcont{All authors contributed equally to this work.}

\author[2]{\fnm{Johannes} \sur{Lengler}}
\equalcont{All authors contributed equally to this work.}

\affil[1]{\orgdiv{LIP6}, \orgname{CNRS, Sorbonne Universit\'e}, \orgaddress{\city{Paris}, \country{France}}}

\affil[2]{\orgdiv{Department of Computer Science}, \orgname{ETH Z\"urich}, \orgaddress{\city{Z\"urich}, \country{Switzerland}}}

\affil[3]{\orgdiv{Department of Computer Science}, \orgname{University of Copenhagen}, \orgaddress{\city{Copenhagen}, \country{Denmark}}}

\abstract{In a seminal paper in 2013, Witt showed that the (1+1) Evolutionary Algorithm with standard bit mutation needs time $(1+o(1))n \ln n/p_1$ to find the optimum of any linear function, as long as the probability $p_1$ to flip exactly one bit is $\Theta(1)$. In this paper we investigate how this result generalizes if standard bit mutation is replaced by an arbitrary unbiased mutation operator. This situation is notably different, since the stochastic domination argument used for the lower bound by Witt no longer holds. In particular, starting closer to the optimum is not necessarily an advantage, and OneMax is no longer the easiest function for arbitrary starting positions. 

Nevertheless, we show that Witt's result carries over if $p_1$ is not too small, with different constraints for upper and lower bounds, and if the number of flipped bits has bounded expectation~$\chi$. Notably, this includes some of the heavy-tail mutation operators used in fast genetic algorithms, but not all of them. We also give examples showing that algorithms with unbounded $\chi$ have qualitatively different trajectories close to the optimum.}




\maketitle

\section{Introduction}
\label{sec:intro}

In real-world optimization, whether of academic or industrial nature, we are often tasked with problems for which no efficient algorithms are readily available -- either because the problems are intrinsically hard or simply because we lack time or other resources to develop problem-specific solution strategies. Facing such problems, practitioners often resort to heuristic approaches. A particularly prominent class of optimization heuristics is that of so-called \emph{randomized search heuristics}, also studied under the name of \emph{stochastic local search algorithms}~\cite{HoosS2004}. These algorithms work in an iterative fashion, alternating between the generation of solution candidates, their evaluation, and using the so-obtained information to adjust the strategy that is used to generate the next candidates. Among the most commonly applied randomized search heuristics are greedy local search strategies, Simulated Annealing~\cite{SA83}, and evolutionary algorithms~\cite{EibenS03}. Understanding the behavior of these search heuristics by mathematical means has been an important driver for the design of state-of-the-art components~\cite{DoerrN21teloSurvey}. 

With this paper, we contribute to the analysis of a key component of heuristic search methods, the so-called \emph{mutation operators}, i.e., the procedure that describes how to generate new solution candidates, referred to as ``\emph{offspring}'', from a single, typically  previously evaluated and often a best-so-far, solution, the \emph{``parent}''. While local search heuristics sample the new candidates within a deterministic, and typically small, neighborhood around the parent, a core ingredient of evolutionary algorithms is to use a random decision where to sample the new solution(s).    
% One of the most crucial ingredients of evolutionary algorithms is the \emph{mutation operator}, i.e., the procedure that describes how to generate offspring from a single parent. 
On the hypercube $\{0,1\}^n$, for a long time the undisputed default was to use \emph{standard bit mutation}, which flips each bit of the parent independently with the same probability. However, this convention has been challenged in the last years; for example via the \emph{fast} mutation operators~\cite{doerr2017fast}, for which the number of flipped bits follows a heavy-tailed distribution. 
The advantages of using heavy-tailed distributions are rather impressive~\cite{DoerrN21teloSurvey}. They are slightly worse for hillclimbing, but the expected runtime deteriorates only by a constant factor that can be chosen close to one. However, they are massively better at escaping local optima. While it takes $e^{\Omega(k\ln k)}$ steps to flip $k$ bits at once with standard bit mutation of mutation rate $\Theta(1/n)$, it only takes $k^{O(1)}$ steps with fast mutation operators. Consequently, they are faster on landscapes with local optima, like the \jump function~\cite{doerr2017fast,antipov2020runtime} and its generalizations~\cite{bambury2022extended,antipov2021effect}.\footnote{When the required jump size $k$ is known in advance, then choosing the mutation rate to be $k/n$ is optimal, as shown in~\cite{doerr2017fast}. The advantage of the fast mutation operator is that $k$ does not need to be known.} 
%, or random MAX-3SAT instances~\cite{antipov2020fast}.
Heavy-tailed distributions can also help on unimodal landscapes like \onemax. For example, the $(1+(\lambda,\lambda))$ GA~\cite{DoerrDE15} was shown to achieve linear expected runtime~\cite{antipov2020fast} when equipped with fast mutation operators, which is asymptotically the best possible for comparison-based unbiased algorithms.

Other benchmarks on which fast mutation operators or other unbiased mutation operators than standard bit mutation have been found to be useful include theoretical benchmarks like \textsc{LeadingOnes}~\cite{ye2019interpolating} and \textsc{TwoMax}~\cite{friedrich2018escaping}, network problems like maximum cut~\cite{friedrich2018heavy,friedrich2018escaping,quinzan2021evolutionary}, minimum vertex cover~\cite{friedrich2018escaping,buzdalov20221+}, maximum independent set~\cite{ye2020benchmarking}, maximum flow~\cite{mironovich2017evaluation} and SAT~\cite{semenov2021evaluating}, landscape classes like submodular functions~\cite{friedrich2018heavy,quinzan2021evolutionary} and random NK-landscapes~\cite{ye2020benchmarking}, multi-objective settings~\cite{doerr2021theoretical,doerr2022first,doerr20221+} and other problems like subset selection~\cite{wu2018dynamic}, the N-queens problem~\cite{ye2020benchmarking}, the symmetric mutual information problem~\cite{friedrich2018heavy,quinzan2021evolutionary} and many more~\cite{ye2020benchmarking,novak2022scheduling,neumann2022evolutionary,klapalek2021car,doerr2022runtime,antipov2021lazy}. Such mutation operators are integrated into standard benchmarking tools like the IOHprofiler~\cite{doerr2018iohprofiler} and Nevergrad~\cite{bennet2021nevergrad}, and they have been used as building blocks for more sophisticated algorithms~\cite{corus2021automatic,corus2021fast,doerr2022stagnation,neumann2022coevolutionary,pavlenko2022asynchronous}.

The large success of non-standard mutation operators raises the desire to analyze which operators are (provably) optimal for a given problem setting. Such questions can be answered in the black-box complexity framework proposed in~\cite{DrosteJW06BBC} (see~\cite{doerr2020complexity} for a survey on the role of black-box complexity for evolutionary computation). Particularly interesting for the study of mutation operators is the unary unbiased black-box complexity model defined in~\cite{lehrewitt2012}. Unary unbiased black-box algorithms create solution candidates by sampling uniformly at random or by selecting one previously evaluated point~$x$ and a search radius~$r$ (both possibly random) and then sampling the solution candidate uniformly at random among all points at Hamming distance~$r$ from~$x$. The unary unbiased black-box complexity of a collection~$\mathcal{F}$ of functions is then the best (over all unary unbiased algorithms) worst-case (over all problem instances in~$\mathcal{F}$) expected runtime. The study of unary unbiased black-box complexities has led to important insights into the limitation of mutation-based algorithms~\cite{doerr2011faster,doerr2012black,lehrewitt2012,doerr2013black,doerr2014reducing,doerr2015unbiased,lehre2019parallel,DDY2020}, which were exploited for the design of faster algorithms such as the $(1+(\lambda,\lambda))$~GA in~\cite{DoerrDE15}. 

% Due to their large success, it is important to understand unbiased mutation operators in general. Unbiased mutation operators were introduced by Lehre and Witt in~\cite{lehrewitt2012} as mutation operators that are invariant under automorphisms of the hypercube. Their study in the context of unbiased black-box complexity has lead to important insights into the limitation of mutation-based algorithms~\cite{doerr2011faster,doerr2012black,lehrewitt2012,doerr2013black,doerr2014reducing,doerr2015unbiased,lehre2019parallel,DDY2020} and to the development of faster algorithms~\cite{DDY2020}, see~\cite{doerr2020complexity} for a survey.

For \onemax, a tight bound for the unary unbiased black-box complexity was proven in~\cite{DDY2020}. It was shown there that the \emph{drift-maximizing} algorithm that at every step chooses the mutation operator that maximizes the expected progress achieves asymptotically optimal expected runtime, up to small lower order terms. Zooming further into this problem for concrete dimensions, Buskulic and Doerr~\cite{BuskulicD21} showed that slightly better performance can be achieved by increasing the mutation rates, i.e., by implementing a more risky strategy that, at several stages that are sufficiently far away from the optimum, flips more bits (in the hope of making more progress and at the cost of a smaller success probability). The approach developed by~\cite{BuskulicD21} was later extended in~\cite{BuzdalovD20dynolea} to compute the optimal mutation rates for the \ooea and the \olea optimizing \onemax. The best \textit{static} unary unbiased mutation operator for the \olea for a number of different combinations of $n$ and $\lambda$ was numerically approximated in~\cite{buzdalov2021optimal}. In particular, it was shown there that the optimal mutation operators are none of the standard choices that are typically used in evolutionary algorithms. These results demonstrate that even for the optimization of \onemax our understanding of optimal mutation operators is rather limited, both in the static and in the dynamic case. 

\textbf{Our Results:} We aim to extend in this work the above-mentioned results to the optimization of a larger class of functions. The first natural extension of \onemax are linear functions, so we primarily focus on these. Our particular goal is to derive tight bounds for the expected runtime of the \ooea equipped with an arbitrary unary unbiased mutation operator. 

To express our main result, we briefly recall from~\cite{doerr2020complexity} that every unbiased mutation operator can be described by a sequence of $n+1$ probabilities $p_0,p_1,\ldots,p_n$ that sum up to one. We thus identify the mutation operator with the sequence $\calD = (p_0,p_1,\ldots,p_n)$ and write \ooeaD for the \ooea that generates its solution candidates using the mutation operator $\mutD$ that first draws an index $i\in [0,n]$ according to the probabilities (i.e., it picks $i$ with probability $p_i$), and then flips a uniformly random set of exactly $i$ bits. Every \ooea equipped with an arbitrary but static unary unbiased mutation operator can be expressed as a \ooeaD. 
We show the following. 
\begin{theorem}
\label{thm:main-short}
Consider the \ooeaD for a distribution $\calD = (p_0,p_1,\ldots,p_{n})$ with mean $\chi$. If $p_1 = \Theta(1)$ and $\chi = O(1)$, then the expected runtime on any linear function on $\{0,1\}^n$ with non-zero weights is
\begin{align}\label{eq:main-intro}
(1\pm o(1))\frac{1}{p_1}\cdot n\ln n.
\end{align}
\end{theorem}
More precise versions of Theorem~\ref{thm:main-short} will be presented in Corollary~\ref{cor:upper-bound} and Theorem~\ref{thm:lower-bound}. In particular, we will show that the lower bound holds for any function with unique global optimum if $p_{n-1} = o(p_1)$. Moreover, the conditions on $p_1$ and~$\chi$ in Theorem~\ref{thm:main-short} can be slightly relaxed. We show that the expected runtime remains unchanged if $\chi^3p_1^{-2}(1-p_0)^{-1} = o(\ln n/\ln \ln n)$, which is probably not tight. However, we also show that the condition is not superfluous either. If $p_1$ becomes too small, or $\chi$ becomes too large, then the behavior of the algorithm starts to change, see Section~\ref{sec:upper-tight} and~\ref{sec:chi}, respectively. 

Theorem~\ref{thm:main-short} can be seen as a generalization of Witt's seminal work~\cite{witt2013tight} on linear functions, where he showed that the expected runtime of the \ooea using standard bit mutation with arbitrary mutation rates $c/n$  is $(1\pm o(1))\frac{e^c}{c} n\ln n = (1 \pm o(1))\tfrac{1}{p_1}n\ln n$, where $p_1$ is the probability that the mutation flips a single bit. Our proof of the upper bounds closely follows his, but we need to adapt his potential function to  account for the fact that the probabilities $p_i$ may follow any distribution.

For the lower bound we follow the proof strategy from~\cite{DDY2020}. In particular, we use the same symmetrized potential $X_t = \min_x\min\{H(x,z),H(x,\bar{z})\}$, measuring the minimal Hamming distance between any already evaluated solution $x$ and the optimum $z$ or to its bit-wise complement $\bar{z}$. %where $\OM(x)$ is the number of one-bits in $x$ and the minimum goes over all previously visited search points $x$. 
We show that for a wide range of values of $X_t$, the drift is maximized either by single-bit flips or by $(n-1)$-bit flips, and with a parent that achieves the minimum in $X_t$. This allows us to compute an upper bound on the drift, and to use the variable drift lower bound from~\cite{DDY2020}. We obtain a lower bound for any function with unique local optimum, but then $p_1$ needs to be replaced by $p_1+p_{n-1}$ in~\eqref{eq:main-intro}. This is not an artifact of our proof, since we give examples showing that the dependence on $p_{n-1}$ is real. For the special case of linear functions, however, we show a lower bound that is independent of $p_{n-1}$. The latter two results can be found in Section~\ref{sec:lower_pn1}.

% For the lower bound we follow the proof strategy from~\cite{DDY2020}. In particular, we use the same symmetrized \onemax potential $X_t = \min_x\min\{\OM(x),n-\OM(x)\}$, where $\OM(x)$ is the number of one-bits in $x$ and the minimum goes over all previously visited search points $x$. We show that for a wide range of values of $X_t$, the drift is maximized either by single-bit flips or by $(n-1)$-bit flips, and with a parent that achieves the minimum in $X_t$. This allows us to compute an upper bound on the drift, and to use the variable drift lower bound from~\cite{DDY2020}. We obtain a lower bound for any function with unique local optimum, but then $p_1$ needs to be replaced by $p_1+p_{n-1}$ in~\eqref{eq:main-intro}. This is not an artifact of our proof, since we give examples showing that the dependence on $p_{n-1}$ is real. For the special case of linear functions, however, we show a lower bound that is independent of $p_{n-1}$. 

Finally, we also show in Section~\ref{sec:domination} that stochastic domination no longer applies when standard bit mutation is replaced by other unary unbiased mutation operators, in the sense that starting closer to the optimum can increase the expected runtime asymptotically. This even holds on \onemax, the function that simply maps each string $x\in\{0,1\}^n$ to $|x|_1=\OM(x)$, the number of ones in this string. As a consequence, non-elitist algorithms may be faster than elitist algorithms on \onemax.

\textbf{Other Related Work.} 
Apart from black-box complexities, only few things are known in general about the class of unbiased mutation operators. Antipov and Doerr~\cite{antipov2021precise} investigated the mixing time on plateaus for the \ooea with arbitrary unbiased mutation operator. 
Doerr and Kelley~\cite{doerr2023fourier} gave precise expressions for the runtime of the \ooea with arbitrary mutation rate on the $\textsc{Needle}$ and $\textsc{BlockLeadingOnes}$ problems.
Lengler~\cite{lengler2019general} studied the \ooea, the \olea, the \moea and the \moga with arbitrary unbiased mutation operators on monotone functions. He found that those algorithms can optimize all monotone functions if the second moment of the number of bit flips is small compared to the first moment, but that they need exponential time on \hottopic functions otherwise. In particular, all heavy-tail distributions lead to exponential runtimes on \hottopic. %Buzdalov and Doerr~\cite{buzdalov2021optimal} provided an algorithm for computing the optimal unbiased mutation operator for each stage on \onemax for small values of $n$ for the \ooea and the \olea.

\section{Preliminaries}
\label{sec:prelim}

% \textbf{Notation.} 
We use the following notation. For $x>0$ let $\ln^+(x) \coloneqq \max\{1,\ln x\}$. 
For $a,b \in \N$ with $a\le b$ we write $[a,b] = \{a,a+1,\ldots,b\}$ and $[b] = [1,b] = \{1,\ldots,b\}$. We write a vector $x\in\{0,1\}^n$ as $x=(x_1,\ldots,x_n)$. The \emph{\onemax value} $\OM(x) \coloneqq \sum_{i=1}^n x_i$ of $x$ is the number of one-bits in $x$. We write $\vec{0}$ and $\vec{1}$ for the vectors in $\{0,1\}^n$ with $\OM(\vec 0) = 0$ and $\OM(\vec 1) = n$, respectively. \emph{With high probability (w.h.p.)} means with probability $1-o(1)$ as $n\to\infty$.

% \textbf{Unary Unbiased EAs.} 
We identify probability distributions $\calD$ on $[0,n]$ with sequences $(p_0,p_1,\ldots,p_n)$ such that $p_i \ge 0$ for all $i\in [0,n]$ and $\sum_{i\in[0,n]} p_i = 1$, where the probability of obtaining $i$ from $\calD$ is $p_i$. We associate to any such distribution $\calD$ the mutation operator $\mutD$ which draws $k$ from $\calD$ and then applies the $\text{flip}_{k}$ operator which flips a uniform random set of exactly $k$ positions. The probability that $\mutD$ flips the $i$-th bit equals $\chi/n$, where $\chi$ is the expected value of $\calD$, as we show in the following lemma. 

\begin{lemma} \label{lem:prbitflip}
For any $\calD$ with mean $\chi$ and every $i\in [n]$, the associated mutation operator satisfies
\begin{align}
    \Pr [\text{\emph{$i$-th bit is flipped}}] = \chi/n. 
\end{align}
\end{lemma}
\begin{proof}
Let $K$ be the number of flipped bits. Using the law of total probability, the probability that the $i$-th bit is flipped is equal to
\begin{align}
    %\Pr [\text{$i$-th bit flips}] & 
    \sum_{k=0}^n \Pr [\text{$i$-th bit is flipped}\mid K=k]\cdot p_k %\cdot \Pr [k \text{bits \dots in round }t]
     =\sum_{k=0}^n \frac{k}{n} \cdot p_k = \frac{\chi}{n}.
\end{align}
\end{proof}
We remark that a special case of $\calD$ is the binomial distribution with expectation $\chi$. This distribution leads to \emph{standard bit mutation}, which can equivalently be obtained by \emph{independently} flipping each bit with probability $\chi/n$. However, this independence only holds in the special case of standard bit mutation, not in general.

For a probability distribution $\calD$ on $[0,n]$, we define the \ooeaD as the elitist $(1+1)$ algorithm which uses $\mutD$ as mutation operator, see Algorithm~\ref{alg:1}. We write $\text{flip}_{k}(x)$ for a bitstring obtained from $x$ by flipping $k$ bits chosen uniformly at random.  Its \emph{runtime} on a function $f$ is the number of fitness evaluations before it finds a global optimum. 
Following the discussion in~\cite{DDY2020,doerr2020complexity}, the class of elitist (1+1) unary unbiased black-box algorithms with static mutation operators coincides exactly with the collection of all \ooeaD with $\calD$ as above. 

We call any population based algorithm that generates offspring exclusively using the operator $\mutD$ a \emph{static unary unbiased algorithm with flip distribution $\calD$}. In particular, such an algorithm is not required to use elitist selection, may access any previously generated search point, and is allowed to generate more than one offspring per generation. By using the adjective static, we emphasize that the distribution $\calD$ may not change throughout the run of the algorithm.
% satisfying the conditions stated above. 
%~\cite{DDY2020,ye2019interpolating}.
% \begin{algorithm}
%   \caption{The \ooeaD for a fixed distribution $\calD$ and maximizing a function $f:\{0,1\}^n\rightarrow \R$.
%     \label{alg:1}}
% %  \begin{algorithmic}[1]
%         Sample $a^{(0)}$ from $\{0,1\}^n$ uniformly at random.\\
%         \For{$t=0,1,2,3,\ldots$}{
%             Sample $k^{(t)}\sim \calD$. \\ % \hfill // Draw environment for this round \\
%             Create $a'\leftarrow \text{flip}_{k^{(t)}}(a^{(t)})$.  \\
%             \If{$f(a')\leq f(a^{(t)})$}{
%             $a^{(t+1)}\leftarrow a'.$}
%             \Else{$a^{(t+1)}\leftarrow a^{(t)}$}
%         }
% %   \end{algorithmic}
% \end{algorithm}
\begin{algorithm2e}[t]
%  \begin{algorithmic}[1]
        Sample $x$ from $\{0,1\}^n$ uniformly at random\;
        \For{$t=0,1,2,3,\ldots$}{
            Sample $k\sim \calD$\; % \hfill // Draw environment for this round \\
            Create $y\leftarrow \text{flip}_{k}(x)$\;  
            \lIf{$f(y)\geq f(x)$}{$x\leftarrow y$} 
            % \text{\textbf{else}} 
            % {$a^{(t+1)}\leftarrow a^{(t)}$}
        }
%   \end{algorithmic}
  \caption{The \ooeaD for a fixed distribution $\calD$ and maximizing a function $f\colon\{0,1\}^n\rightarrow \R$.}
  \label{alg:1}
\end{algorithm2e}

With \emph{linear functions} we always refer to functions $f:\{0,1\}^n \to \R; \ f(x) = \sum_{i=1}^n w_i x_i$ for non-zero weights $w_i\in \R$. By unbiasedness of the \ooeaD, we may (and will) assume that the weights are positive and sorted, $0 < w_1\le \ldots \le w_n$. For the analysis of linear functions it is convenient to consider minimization instead of maximization, even though both versions are equivalent. We will specify in the proofs when we adapt this perspective, but by default we consider maximization.

In the following, we briefly recall the mathematical tools needed for our analysis. 

\subsection{Drift Analysis}

As it is the case for Witt's result~\cite{witt2013tight}, our upper bound heavily relies on potential function arguments, which are converted into upper bound using the multiplicative drift theorem. 

\begin{theorem}[Multiplicative Drift Theorem~\cite{doerr2012multiplicative}] 
\label{mdt}
Let $S\subset \R$ be a finite set with minimum~$1$. Let $(X_t)_{t\geq 0}$ be a sequence of random variables over $S \cup \{0\} $. Let $T \coloneqq \min\{t \ge 0 \mid X_t = 0\}$ be the hitting time of $0$. Suppose that there is a real number $\delta > 0 $ such that \begin{align}
    \E \big[X_t- X_{t+1} \mid X_t= s \big] \geq \delta s
\end{align}
for all $s\in S $ and all $t\geq 0$ with $\Pr [X_t=s]>0$. Then, for all $s_0\in S$ with $\Pr [X_0=s_0]>0$, \begin{align}
    \E [T \mid X_0 = s_0 ] \leq \frac{\ln (s_0)+1}{\delta}. 
\end{align}
Moreover, for all $r > 0$,
\begin{align}
    \Pr \left[ T > \frac{\ln(s_0)+ r}{\delta}\right] \leq e^{-r}.
\end{align}
\end{theorem}
Note that we state the theorem requiring the set $S$ to have minimum $1$, whereas other variants introduce an explicit minimum $s_{\min}$. In the proof of the lower bound we will apply the following lower bound for variable drift~\cite[Theorem~9]{DDY2020}.
\begin{theorem}[Variable Drift, lower bound] \label{thm:discrete_var_lb}
    Let $N\in \N$ and let $(X_t)_{t\geq 0}$ be a sequence of non-increasing random variables over $[0,N]$, i.e., it holds $\Pr [X_t \leq X_{t-1}]=1$ for all $t>0$, and let $T \coloneqq \min\{t \ge 0 \mid X^{(t)} = 0\}$ be the hitting time of $0$. 
    Suppose that there are two functions $c : [N]\rightarrow [0,N]$ and monotonically increasing $h : [0,N]\rightarrow \R^+_0$, and a constant $0\leq p < 1$ such that 
    \begin{enumerate}
        \item $X_{t+1} \geq c(X_t)$ with probability at least $1-p$ for all $t<T$, and
        \item $\E \left[X_t -X_{t+1} \mid X_t \right]\leq h(X_t) $ for all $t<T$.
    \end{enumerate}
    Let $\mu :[0,N] \rightarrow [0,N]$ be defined by $\mu (x) \coloneqq\max \{i \mid c(i) \leq x\}$, and let $g :[0,N]\rightarrow \R^+_0$ be defined by $g(x)\coloneqq \sum_{i=0}^{x-1}\frac{1}{h(\mu(i))}$. Then 
    \begin{align}
        \E \left[T \mid X_0\right] \geq g(X_0) - \frac{g^2(X_0)p}{1+g(X_0)p}.
    \end{align}
\end{theorem}

\subsection{Concentration Bounds} 
In the proofs for Section~\ref{sec:lower}, we make use of the following additive and multiplicative Chernoff bounds, originally shown by Hoeffding~\cite{hoeffding1963probability}; see~\cite[Section~2.3]{DDY2020} and~\cite[Section~10]{doerr2020probabilistic} for these formulations. 

\begin{theorem}[Additive Chernoff Bound] \label{additive_chernoff}
    Assume that $X$ is a hypergeometrically distributed random variable with parameters $N, n, m$, or let $X $ be a sum of $n$
    independent random variables $X_1, \ldots , X_n$, with each taking values in $\{0,1\}$. Then we have for all $\varepsilon> 0$
    \begin{align}
        &\Pr \left[X \geq \E [ X ]+ \varepsilon\right] \leq e^{-2 \varepsilon^2/n}, \quad \text{and} 
        \\ & \Pr \left[X \leq \E [ X ]- \varepsilon\right] \leq e^{-2 \varepsilon^2/n}
    \end{align}
\end{theorem}

\begin{theorem}[Multiplicative Chernoff Bound] \label{multiplicative_chernoff}
    Assume that $X$ is a hypergeometrically distributed random variable with parameters $N, n, m$. Then we have for all $\delta> 0$,
    \begin{align}
        \Pr [X \geq (1+ \delta) \E [ X ] ] \leq \bigg( \frac{e^\delta}{(1+\delta)^{1+\delta}}\bigg)^{\E [X]}.
    \end{align}
\end{theorem}

\section{Upper Bounds}
\label{sec:upper}

We first note the following, simpler version of the upper bound stated in Theorem~\ref{thm:main-short} for \onemax,  which does not require any assumption on the distribution $\calD$. It straightforwardly follows from the standard multiplicative drift theorem (Theorem~\ref{mdt}), applied to the lower bound on the drift obtained by considering only 1-bit flips.
\begin{theorem}\label{thm:upperOM}
    Let $\calD=(p_0,p_1,\ldots,p_n)$ be a probability distribu\-tion on $[0,n]$. The runtime of the \ooeaD on \onemax is at~most
    \begin{align}
        (1 \pm o(1)) \frac{1}{p_1} n \ln n
    \end{align}
    in expectation and with high probability.
\end{theorem} 
\begin{proof}
    We consider $X_t\coloneqq n- \OneMax(x^{(t)})$. We have \begin{align}
        \E [X_t-X_{t+1}\mid X_t = s] \geq p_1 \cdot \frac{s}{n}.
    \end{align}
    By Theorem~\ref{mdt}, we have \begin{align}
        \E [T \mid X_0] \leq \frac{\ln n +1}{p_1/n}= (1 \pm o(1))\frac{1}{p_1} 
         n \ln n.
    \end{align}
    Taking $r\coloneqq \ln \ln n$ in Theorem~\ref{mdt} concludes the proof. 
\end{proof}

The key ingredient for generalizing the bound from \onemax to all linear functions as in Theorem~\ref{thm:main-short} is the following theorem, which generalizes \cite[Theorem~4.1]{witt2013tight} to the \ooeaD with (almost) arbitrary $\calD$. Our proof follows~\cite{witt2013tight}, with the following differences: 
First, we noted a mistake in the proof of the upper bound in~\cite{witt2013tight}. Equation (4.2) there does not hold for the events $A_i$ as defined in~\cite{witt2013tight}. We thank Carsten Witt for providing the following fix upon our inquiry (personal communication): By conditioning the events $A_i$ on the event that the offspring is accepted, equation (4.2) holds as in that case, the expectation is zero if none of the $A_i$ occur.
%(see \eqref{deltaLR} in the appendix for details). 
Furthermore, the inequality (4.3) in~\cite{witt2013tight} still holds, which can be shown by applying Bayes' theorem and linearity of expectation. Claim~\ref{lem:bounddelR} isolates the details of this fix. 

Apart from this issue, the biggest challenge was to adapt the potential function used in~\cite{witt2013tight}, since we need to deal with arbitrary unbiased mutation operators. In particular, our potential involves the quantities $\chi$ and $p_1$. With the modified potential, we can show the following generalization of~\cite[Theorem~4.1]{witt2013tight}.

\begin{theorem} \label{thm:upperbnd}
Let $\calD=(p_0, p_1, \ldots, p_n)$ be a probability distribution on $[0,n]$ with expectation $\chi$ and with $p_1 > 0$. For any arbitrary $\alpha >0$, and $r>0$, the runtime of the \ooeaD on any linear function on $n$ variables is at most 
\begin{align}\label{eq:thm:upperbound}
  b(r) \coloneqq \frac{n}{p_1}\cdot \frac{\alpha}{\alpha -1} \cdot  \left(  \frac{\alpha n\chi^3}{(n-1)p_1^2} + \ln \left(\frac{(n-1)p_1^2}{\chi^3} \right) + r\right) %\nonumber \\
%  & \eqqcolon b(r)
\end{align} 
with probability at least $1-e^{-r}$, and it is at most $b(1) $ in expectation. 
\end{theorem}

To ease the comparison of our proof of Theorem~\ref{thm:upperbnd} and Theorem~4.1 in~\cite{witt2013tight}, we keep the same notation. 
For the parts of the proof in \cite{witt2013tight} that transfer directly to our case, we will simply cite them. We note that in the proof of Theorem~\ref{thm:upperbnd}, we assume that the algorithm is minimizing the objective function, rather than maximizing it. This equivalent perspective is more common in the runtime analysis of evolutionary algorithms on linear functions, as it is more convenient to work with.  

% Now we can turn to Theorem~\ref{thm:upperbnd}. The proof is adapted from~\cite{witt2013tight}.
\begin{proof}[Proof of Theorem~\ref{thm:upperbnd}] 
Let $f : \{0,1\}^n \rightarrow \R$, $f(x) = w_1 x_1 + \ldots + w_n x_n$, with
%that all weights are 
%non-zero,
%since otherwise it does not matter how this bit is set, which decreases the runtime. 
%Moreover, we may assume that the weights are 
$0 < w_1\le \ldots \le w_n$. The proof works by applying the multiplicative drift theorem to a carefully chosen potential. To this end, following~\cite{witt2013tight}, we define a new (linear) function $g$, and consider the stochastic process $X_t= g(x^{(t)})$, where $x^{(t)}$ is the current search point of the \ooeaD at time $t$. The weights $g_i$ of the function $g$ are as follows. For all $1\leq i \leq n$, we let 
\begin{align}\label{eq:def-gammai}
    \gamma_i \coloneqq \left(1+ \frac{\alpha \chi^3}{(n-1)p_1^2} \right)^{i-1},
\end{align}
put $g_1\coloneqq \gamma_1 = 1$, and for $2\leq i \leq n$ we set 
\begin{align}\label{eq:def-gi}
    g_i \coloneqq \min \left\{\gamma_i, g_{i-1} \frac{w_i}{w_{i-1}} \right\}\geq 1.
\end{align}
We note that the $g_i$ are non-decreasing with respect to $i$,
%, as $w_i/w_{i+1}\geq 1, $ and the base of the exponentiation in the definition of $\gamma_i$ is at least $1$. 
define $g(x) \coloneqq g_1 x_1 + \ldots + g_n x_n$, and let $X_t \coloneqq g(x^{(t)})$. 
%Consider the stochastic process $X_t= g(x^{(t)})$, where $x^{(t)}$ is the current search point of the \ooeaD at time $t$. 
Then $X_t=0$ if and only if $f$ has been optimised. Let $\Delta_t \coloneqq X_t- X_{t+1}$. First, we will show that \begin{align} 
    \E \left[\Delta_t \mid X_t = s \right] \geq s\cdot \frac{p_1}{n}\cdot \frac{\alpha-1}{\alpha}. \label{driftbnd}
\end{align}

In the following, we recall some notation from \cite{witt2013tight}. Let $s\geq 1$ be an arbitrary non-zero value of $g$ and fix a search point $x^{(t)}$ with $g(x^{(t)})=s$. From now on, we implicitly assume that $X_t=s$. Let $I\coloneqq \{i \mid x_i^{(t)}=1 \}$ be the index set of the one-bits in $x^{(t)}$ and $Z \coloneqq [n]\setminus I$ be the zero-bits.
%We can assume that $I \neq \varnothing$, as otherwise the function is already optimised. 
Let $x'$ denote the random offspring generated by the \ooeaD from $x^{(t)}$. We denote by $I^*\coloneqq \{i \in I \mid x'_i= 0\} $ the index set of one-bits that are flipped, and by $Z^* \coloneqq \{i\in Z \mid x'_i = 1\}$ the zero-bits that are flipped. Let $k(i) \coloneqq \max \{j\leq i \mid g_j = \gamma_j\}$ for all $i\in I$. Note that $k(i)\geq 1$. Set $L(i)\coloneqq [k(i),n]\cap Z$ and $R(i)\coloneqq [1,k(i)-1]\cap Z$. Note that $R(i)$ and $L(i)$ are the indices on the right (inclusive) and left of $k(i)$, respectively, when considering the indices to be ordered from $n$ to $1$.

For $i \in I$, we define the events $A_i$ as \begin{enumerate}
    \item $i= \max_{j\in I^*} j$, i.e., it is the leftmost one-bit that is flipped (w.r.t. the ordering above), and 
    \item $\sum_{j\in I^* } w_j-\sum_{j\in Z^* } w_j\geq 0 $, i.e. the offspring is accepted. 
\end{enumerate}
Note that the $A_i$ are mutually disjoint. Furthermore, if none of the $A_i$ occur we have $\Delta_t=0$, as then the offspring is either rejected or equal to the parent. Let 
\begin{align}
    \Delta_L(i) & \coloneqq \sum_{j\in I^* } g_j-\sum_{j\in Z^*\cap L(i) } g_j , \text{ and} \\
    \Delta_R(i) & \coloneqq - \sum_{j\in Z^*\cap R(i) } g_j. 
\end{align}
Conditioning on $A_i$, we have that $\Delta_t = \Delta_L(i) + \Delta_R (i)$, as in that case, the offspring is accepted. By the law of total expectation,
\begin{align}
    \E [\Delta_t ] & = \sum_{i\in I}  \E [\Delta_t \mid A_i ] \cdot \Pr [A_i]  +
    \underbrace{\E \bigg[\Delta_t \mid \overline{\bigcup_{i\in I} A_i }\bigg]}_{=0} \cdot \Pr \bigg[\overline{\bigcup_{i\in I} A_i }\bigg]\nonumber\\
    & = \sum_{i\in I}  \E \left[\Delta_L(i) + \Delta_R(i) \mid A_i \right] \cdot \Pr \left[A_i\right] \label{deltaLR} \\ 
    & = \sum_{i\in I} \big( \E \left[\Delta_L(i)\mid A_i \right] \cdot \Pr \left[A_i\right]+\E \left[\Delta_R(i)\mid A_i \right] \cdot \Pr \left[A_i\right]\big),\nonumber 
\end{align}
where we used linearity of conditional expectation for the last step. 
%This shows that (4.2) from \cite{witt2013tight} is indeed true for the new definition of $A_i$.

In the following, we want to estimate the terms appearing in the above sum. First, we turn our attention to $\E[\Delta_L(i)\mid A_i ]\cdot \Pr[A_i]$. It follows from the same argument as \cite{witt2013tight} uses to show the non-negativity
of $\Delta_L(i)$ 
%conditioned on the old $A_i$
that $\E[\Delta_L(i)\mid A_i ]\geq 0$. As in \cite{witt2013tight}, defining the event $S_i \coloneqq \{Z^* \cap L(i)= \emptyset\}$, applying the law of total expectation to $A_i\cap S_i$ and $A_i\cap \overline{S_i}$, and noting that $\E [\Delta_L(i)\mid A_i \cap \overline{S_i}]\geq 0$, we get 
\begin{align}
    \E [\Delta_L(i)\mid A_i ] \Pr [A_i] 
     & \geq \E [\Delta_L(i)\mid A_i \cap S_i]\cdot \Pr [A_i \cap S_i] \nonumber \\
     &
    \geq g_i \cdot \Pr [A_i \cap S_i].
\end{align}

For the second inequality, recall that, assuming $A_i$, $i\in I^*$. To lower bound $\Pr [A_i \cap S_i]$, we note that $A_i\cap S_i$ occurs if we flip exactly the $i$-th bit. The probability of flipping exactly the $i$-th bit, which is $p_1/n$, is thus a lower bound for $\Pr [A_i \cap S_i]$, yielding
\begin{align} \label{align:deltaL}
     \E [\Delta_L(i)\mid A_i ] \cdot  \Pr [A_i]&  \geq  g_i \cdot \frac{p_1}{n}.
\end{align}

In the following claim, we estimate $ \E \left[\Delta_R(i ) \mid A_i \right]$. 
%The role of this lemma corresponds to the estimate in equation (4.3) in \cite{witt2013tight}. By the same argument as in the proof of the Lemma below, one can show that (4.3) holds for the new $A_i$'s.
\begin{claim}\label{lem:bounddelR}
We have \begin{align}
    \E \left[\Delta_R(i ) \mid A_i \right] \geq - \frac{\chi^2}{(n-1) p_1} \sum_{j\in R(i) } g_j . 
\end{align}
% Below is the stronger (but now unnecessary) version of this Lemma.
% \begin{lemma}\label{lem:bounddelR}
% We have \begin{align}
%     \E \left[\Delta_R(i ) \mid A_i \right] \geq - \frac{\chi (\chi - (1-p_0))}{(n-1) p_1} \sum_{j\in R(i) } g_j . 
% \end{align}
\end{claim}
\begin{proof}
It holds
\begin{align}
\begin{split}
    \E [\Delta_R(i ) \mid A_i ] & = \E \bigg[- \sum_{j\in Z^*\cap R(i) } g_j \mid A_i \bigg] 
    =   \E \bigg[- \sum_{j\in R(i) } \mathbbm{1}_{\{j\in Z^* \}}g_j \mid A_i \bigg] \\
    %& = - \sum_{j\in R(i) } g_j \E \left[\mathbbm{1}_{\{j\in Z^* \}} \mid A_i \right] \\
    & = - \sum_{j\in R(i) } g_j \Pr \left[\{j\in Z^* \}\mid A_i \right], \label{bounddelR}
\end{split}
\end{align}
where we used linearity and the definition of conditional expectation. Since for all $i\in I$, $\Pr[A_i]> 0 $ (it is possible to flip just the $i$-th bit, as $p_1 > 0$), we can apply Bayes' theorem to the conditional probability above, yielding 
\begin{align}\label{bounddelR2}
    \Pr [\{j\in Z^* \}\mid A_i ] = \Pr [\{j \in Z^* \}]\cdot \frac{\Pr [A_i \mid \{j\in Z^* \}]}{\Pr[A_i]}.
\end{align}
As $j\in R(i)$, we have that $j\in Z$, so $j \in Z^*$ if and only if the $j$-th bit is flipped. Hence, by Lemma~\ref{lem:prbitflip}, $\Pr [\{j \in Z^* \}] = \chi/n$. 
We lower bound the denominator by $p_1/n$. The numerator is at most
\begin{align}
 \Pr [\text{$i$-th bit is flipped} \mid\{j\in Z^* \}], 
\end{align}
as it is necessary to flip the $i$-th bit for $A_i$ to occur. 
% \begin{align}
%   \Pr & \bigg[\{\text{$i$-th bit flips}\}\cap\{I^*\cap \{i+1, \ldots n\}=\varnothing \}\\ & \qquad\cap  \left\{\sum_{j\in I^* } w_j-\sum_{j\in Z^* } w_j\geq 0 \right\} \mid \{j\in Z^* \} \bigg],
% \end{align}
We calculate using the law of total probability\begin{align}
   & \Pr [\text{$i$-th bit is flipped} \mid\{j\in Z^* \}]  = \Pr [\text{$i$-th bit is flipped} \mid \text{$j$-th bit is flipped}]\nonumber
    \\& \quad  = \sum_{k=2}^n  \Pr [\text{$i$-th bit is flipped} \mid \text{$j$-th bit is flipped}\cap \{\text{$k$ bits flip in total}\}] \cdot p_k \nonumber\\
    & \quad = \sum_{k=2}^n \frac{k-1}{n-1} \cdot p_k  = \frac{1}{n-1} \left(\sum_{k=1}^n (k-1) p_k\right)\leq \frac{\chi}{n-1}.
    %\\& \quad = \frac{1}{n-1} \left(\sum_{k=1}^n k p_k -\sum_{k=1}^n p_k \right) = \frac{\chi - (1-p_0)}{n-1}.
\end{align}
Plugging the bounds obtained above into~\eqref{bounddelR2} and~\eqref{bounddelR} completes the proof of Claim~\ref{lem:bounddelR}. 
\end{proof} 
We now continue with the proof of Theorem~\ref{thm:upperbnd}. By Lemma~\ref{lem:prbitflip}, we have $\Pr[A_i]\leq \Pr[i\text{-th bit is flipped}]=\chi/n$, as the $i$-th bit needs to flip for $A_i$ to occur. Plugging this, Lemma~\ref{lem:bounddelR}, and \eqref{align:deltaL} into \eqref{deltaLR}, we get 
\begin{align}
\begin{split}
    \E[\Delta_t] & \geq \sum_{i\in I} \bigg(g_i \cdot \frac{p_1}{n} - \frac{\chi}{n} \cdot \frac{\chi^2}{(n-1)p_1}\sum_{j\in R(i) } g_j\bigg) \\
    & \geq \sum_{i\in I} \bigg(\frac{p_1}{n } \cdot \frac{g_i}{g_{k(i)}}\cdot \gamma_{k(i)} - \frac{\chi^3}{n(n-1)p_1} \sum_{j=1}^{k(i)-1} \gamma_j \bigg).  \label{35}
\end{split}
\end{align}
We bound the summand for $i\in I$ in \eqref{35} below, using the geometric sum formula
\begin{align}
   & \frac{p_1}{n} \cdot \frac{g_i}{g_{k(i)}} \left(1+ \frac{\alpha \chi^3}{(n-1)p_1^2} \right)^{k(i)-1} 
  - \frac{\chi^3}{n(n-1)p_1}\cdot  \frac{\left(1+ \frac{\alpha \chi^3}{(n-1)p_1^2} \right)^{k(i)-1}-1}{\frac{\alpha \chi^3}{(n-1)p_1^2}}
  \\ & \quad \geq \left(1- \frac{1}{\alpha }\right)\cdot \frac{p_1}{n} \cdot \frac{g_i}{g_{k(i)}}\cdot \left(1+ \frac{\alpha \chi^3}{(n-1)p_1^2} \right)^{k(i)-1}  = 
    \frac{\alpha-1}{\alpha}\cdot \frac{p_1}{n} \cdot g_i.
\end{align}
% Using the geometric sum formula, we can calculate the term for~$i$: 
% \begin{align}
%    & \frac{p_1}{n} \cdot \frac{g_i}{g_{k(i)}} \left(1+ \frac{\alpha \chi^3}{(n-1)p_1^2} \right)^{k(i)-1} \\
%   & \qquad  - \frac{\chi^2 (\chi-(1-p_0))}{n(n-1)p_1}\cdot  \frac{\left(1+ \frac{\alpha \chi^3}{(n-1)p_1^2} \right)^{k(i)-1}-1}{\frac{\alpha \chi^3}{(n-1)p_1^2}}
%   \\ & \quad \geq \left(1- \frac{1}{\alpha }\right)\cdot \frac{p_1}{n} \cdot \frac{g_i}{g_{k(i)}}\cdot \left(1+ \frac{\alpha \chi^3}{(n-1)p_1^2} \right)^{k(i)-1}
%   \\ & \quad = \left(1- \frac{1}{\alpha }\right)\cdot \frac{p_1}{n} \cdot g_i,
% \end{align}
% where the inequality bounds the $-1$ at the very end by $0$ and uses $g_i \geq g_{k(i)}$.
%Note that since $0<p_1\neq 1$, the denominator of the double fraction above is different from zero. %note the issue with the denominator of the double fraction in case of RLS.
Plugging this back into \eqref{35} gives us \begin{align}\label{eq:drift-lower-bound}
    \E[\Delta_t]\geq \sum_{i\in I } \frac{\alpha-1}{\alpha}\cdot \frac{p_1}{n} \cdot g_i = \frac{\alpha-1}{\alpha}\cdot \frac{p_1}{n} \cdot g(x^{(t)}),
\end{align}
which shows \eqref{driftbnd}.

Finally, we apply the multiplicative drift theorem (Theorem~\ref{mdt}) to finish the proof. To this end, we compute \begin{align}
    X_0\leq \sum_{i=1}^n g_i \leq \sum_{i=1}^n \gamma_i = \frac{\left(1+ \frac{\alpha \chi^3}{(n-1)p_1^2} \right)^{n}-1}{\frac{\alpha \chi^3}{(n-1)p_1^2}}\leq \frac{e^{n\cdot\frac{\alpha \chi^3}{(n-1)p_1^2} }}{\frac{\alpha \chi^3}{(n-1)p_1^2}}.
\end{align}
Hence,\begin{align}
    \ln  X_0\leq n\cdot\frac{\alpha \chi^3}{(n-1)p_1^2} + \ln \left(\frac{(n-1)p_1^2}{ \chi^3} \right),
\end{align}
as $-\ln \alpha < 0$ (because $\alpha > 1$). We apply Theorem~\ref{mdt} with $\delta = ((\alpha-1)/\alpha) \cdot (p_1/n)$, which concludes the proof of Theorem~\ref{thm:upperbnd}. 
\end{proof}

From Theorem~\ref{thm:upperbnd} we obtain the following upper bound, which relaxes the conditions on $p_1$ and $\chi$ in Theorem~\ref{thm:main-short} and shows that the bound holds not only in expectation but also w.h.p.
\begin{corollary}\label{cor:upper-bound} %fix the numbering
Let $\calD=(p_0, p_1, \ldots, p_n)$ be a probability distribution on $[0,n]$ with expectation $\chi$. 
Assume that $p_1>0$ and $\chi^3p_1^{-2}(1-p_0)^{-1} = o(\ln n/ \ln \ln n)$. 
%, and $\ln (p_1^2/\chi^3)=o(\ln n)$.
    %\text{ and}\\
     %\ln\left(\frac{p_1^2}{\chi^2(\chi- (1-p_0))}\right)& = O(\ln n).
Then the runtime of the \ooeaD on any linear function is at most 
\begin{align}
    (1 + o(1)) \frac{1}{p_1}\cdot n \ln n
\end{align}
in expectation and with high probability.
\end{corollary}
Note that since $1-p_0 \ge p_1$, we could replace the requirement $\chi^3p_1^{-2}(1-p_0)^{-1} = o(\ln n/ \ln \ln n)$ by the stronger requirement $\chi^3/p_1^3 = o(\ln n/ \ln \ln n)$. In particular, this is trivially satisfied if $p_1=\Theta(1)$ and $\chi=O(1)$, as required in Theorem~\ref{thm:main-short}.

\begin{proof}%[Proof of Corollary~\ref{cor:upper-bound}]
We first treat the case $p_0 =0$, which implies $\chi \ge 1$. Let $\alpha \coloneqq \ln \ln n$. As in \cite{witt2013tight}, $\alpha/(\alpha-1) = 1+o(1)$, and we may bound $\ln ((n-1)p_1^2/\chi^3)\leq \ln n$, since $p_1\leq 1$ and $\chi \ge 1$. Moreover, we observe $\alpha n\chi^3/((n-1)p_1^2) = o(\ln n)$. 
Thus $b(r)$ in Theorem~\ref{thm:upperbnd} is at most
\begin{align}
    \frac{n}{p_1} (1+o(1))\left(o(\ln n)+\ln n +r\right)
   %& O(n \ln \ln n) \cdot \frac{\chi^3}{p_1^3} + (1 \pm o(1))\cdot \frac{n}{p_1} \left( \ln (n)  + r\right).
\end{align}
Taking $r\coloneqq 1 $, we get the claimed expected runtime, and with $r\coloneqq \ln \ln n$, it follow that the bound holds w.h.p.

Now we turn to $p_0 > 0$. In this case, we define an auxiliary distribution $\calD' = (p_0',p_1',\ldots,p_n')$ by $p_0' \coloneqq 0$ and $p_i' \coloneqq p_i/(1-p_0)$ for $i \ge 1$. In other words, $\calD'$ is the same as the distribution $\calD$ conditioned on not drawing $0$. It has expectation $\chi' \coloneqq \chi/(1-p_0)$. Therefore,
\begin{align}
\frac{(\chi')^3}{(p_1')^2(1-p_0')} = \frac{\chi^3}{p_1^2 (1-p_0)} = o\left(\frac{\ln n}{\ln \ln n}\right).   
\end{align}
Hence, $\calD'$ is covered by the case that we have already treated. Thus, the runtime of the $(1+1)$-\text{EA}$_{\calD'}$ is at most
\begin{align}\label{eq:remove-p0}
    (1+o(1))\frac{1}{p_1'}\cdot n\ln n = (1+o(1))\frac{1-p_0}{p_1}\cdot n\ln n,
\end{align}
in expectation and with high probability.

Note that no-bit flips are just idle steps of the \ooeaD, therefore the \ooeaD and the $(1+1)$-EA$_{\mathcal D'}$ follow exactly the same trajectory through the search space, except that the \ooeaD performs idle steps with probability $p_0$. Hence, we can couple both algorithms such that they behave identical in non-idle steps. Note that the expected time until a non-idle step is geometrically distributed with expectation $1/(1-p_0)$. Therefore, if the $(1+1)$-EA$_{\mathcal D'}$ finds the optimum in $T$ steps, then the number of steps of the \ooeaD is the sum of $T$ independent geometrically distributed random variables with expectation $1/(1-p_0)$. By Wald's equation, the \ooeaD thus needs $\E[T]/(1-p_0)$ steps in expectation, and $(1+o(1))\frac{1}{(1-p_0)}\E[T]$ steps with high probability, see~\cite[Theorem~1.10.32]{doerr2020probabilistic}. Together with~\eqref{eq:remove-p0}, this implies the claim.
\end{proof}

Under some less restrictive conditions, we can give a polynomial upper bound on the expected runtime.

\begin{corollary}\label{cor:polynomial_ub}
    Let $\calD=(p_0, p_1, \ldots, p_n)$ be a probability distribution on $[0,n]$ with expectation $\chi$. The runtime of the \ooeaD on any linear function is $O(n\chi^3/p_1^3 + n\ln n /p_1)$ in expectation and with high probability. In particular, this expression is $O(n^4/p_1^3)$, and thus polynomial in $n$ if $1/p_1$ is polynomial in~$n$.
\end{corollary}

\begin{proof}
    We take $\alpha \coloneqq 2$,
    %First, we assume again that $p_0=0$.
    and apply Theorem~\ref{thm:upperbnd}. Noting that $\chi \geq p_1$ and thus $(n-1)p_1^2/\chi^3 \le n/p_1$, by~\eqref{eq:thm:upperbound} we can bound the runtime of the \ooeaD on any linear function by
    %, and it follows by the same consideration as above that the runtime of the \ooeaD on any linear function is at most\duri{TO DO : adjust the proof here. (Again $\chi$-trick.)}
    \begin{align}
        \frac{2n}{p_1} \left( \frac{2n\chi^3}{(n-1) p_1^2 } + \ln\left( \frac{n}{p_1}\right)+ r \right)= O\left(\frac{n \chi^3}{p_1^3} + \frac{n\ln n}{p_1} + r\right),
    \end{align}
    where the bound holds in expectation for $r=1$, and with high probability for e.g. $r =\ln n$. This proves the first statement. The second statement holds since $\chi \leq n$.
    %, which concludes the proof of Corollary \ref{cor:polynomial_ub}.
\end{proof}

%%%%%%%%%%%%%%%%%%
%%%% LOWER
%%%%%%%%%%%%%%%%%%
\section{Lower Bound}
\label{sec:lower}

The following theorem is the main result shown in this section.

\begin{theorem}\label{thm:lower-bound}
    The expected runtime of any static unary unbiased algorithm with flip distribution $\calD~=~(p_0,p_1,\dots,p_n)$ satisfying $p_1+p_{n-1}= n^{-o(1)}$ on any function $f: \{0,1\}^n\rightarrow \R$ with unique global optimum is at least 
    \begin{align}\label{eq:thm:lowerbound}
        (1- o (1)) \frac{1}{p_1+p_{n-1}} n \ln n.
    \end{align}
\end{theorem}

Note that most common mutation operators satisfy $p_{n-1} = o(p_1)$, in which case~\eqref{eq:thm:lowerbound} simplifies to $(1- o (1)) \tfrac{1}{p_1} n \ln n$.
We remark that a coarser lower bound of $\Omega(n \ln n)$ follows from \cite{lehrewitt2012} and \cite{DDY2020}. However, in contrast to \cite{lehrewitt2012}, we are interested in understanding the leading constant, and in contrast to \cite{DDY2020}, we are interested in the expected runtime for \emph{static} unary unbiased distributions. A common technique to prove lower bounds that apply to any function from some problem collection is to bound the expected runtime of the algorithm on \onemax and to show that \onemax is the ``easiest'' among all functions from the collection. ``Easiest'' here means in particular that the expected runtime of the algorithm optimizing a given function from the set cannot be smaller than its expected runtime on \onemax, but easiest can also have the stronger meaning of stochastic domination~\cite{witt2013tight,jorritsma2023comma}. In many cases, e.g., when considering the \ooea with standard bit mutation, \onemax is the easiest among all functions with unique global optimum, as was first shown in~\cite{doerr2012multiplicative} for mutation rate $p=1/n$ and then in~\cite{witt2013tight} more generally for all (static or dynamic) mutation rates $p\le 1/2$. However, in our situation it is not true that \onemax is the easiest function, as we will discuss in Section~\ref{sec:domination}. 

Our proof for Theorem~\ref{thm:lower-bound} follows the strategy used in~\cite{DDY2020}. In particular, we apply their lower bound theorem for variable drift~\cite[Theorem~9]{DDY2020} in the same way. We quote their Lemma~13 directly, and the proof of our Theorem~\ref{thm:lwr_bnd} below differs from the proof of  their Theorem~14 only in the calculations and bounds used. The key difference between their proof and ours is that we use a different function $h$ to bound the expected change in the potential. Most of the work goes into showing that this function $h$ is indeed applicable, and providing an upper bound on its values that allows us to translate the result of Theorem~\ref{thm:lwr_bnd} into the asymptotic formulation of Theorem~\ref{thm:lower-bound}. 

To implement the proof strategy of~\cite{DDY2020}, we use the same potential function to measure the progress of the optimization process. That is, we denote by $(x^{(0)}, x^{(1)}, \dots, x^{(t)})$ the sequence of the first $t+1$ search points evaluated by the algorithm and we denote by $v_t\in \{x^{(0)}, \dots, x^{(t)}\}$ the parent chosen by the algorithm in iteration~$t$. We define the potential at time $t$ as 
\begin{align}
    X_t\coloneqq \min_{0\leq i \leq t} d\big(x^{(i)}\big),
\end{align}
where $d$ is the distance function 
\begin{align}
    d(x) \coloneqq \min \{{ H(x,z), H(x,\bar{z})}\},
       % d(x) \coloneqq \min \{{ n - \OM (x) , \OM(x)}\}.
\end{align}
where $z$ is the optimum and $\bar{z}$ its bit-wise complement $(1,\ldots,1)-z$.
The reason for considering the symmetric distance to the optimum and its complement is that an optimal unary unbiased black-box algorithm may first reach $\bar{z}$, and then flip all bits at once. Furthermore, as we will show in Lemma~\ref{lem:pn-1}, it is possible to make progress towards the optimum in a way that can be measured in terms of~$d$, while the Hamming-distance to the optimum increases. Note that the proofs in~\cite{DDY2020} consider only the Hamming distance to $\vec{0}$ and $\vec{1}$. This is justified by the fact that the location of the optimum does not make any difference for \ooeaD-type algorithms which use so-called unbiased mutation operators. This is a standard argument in the analysis of evolutionary algorithms and other randomized search heuristics, see also footnote~1 in~\cite{DDY2020}. To ease the comparison with~\cite{DDY2020}, we therefore assume in the following, without loss of generality, that $z=(1,\ldots,1)$ and we make use of the fact that, with this assumption, 
\begin{align}
    d(x) = \min \{{ n - \OM (x) , \OM(x)}\}.
\end{align}   
 
Note that the sequence $\left(X_t\right)_{t\geq 0}$ is non-increasing, so we may apply 
% Theorem~\ref{thm:discrete_var_lb} 
the variable drift lower bound from~\cite[Theorem~9]{DDY2020} to it.

Next, we define the function $\Tilde{h}$, which gives the precise drift in the case where the algorithm uses a bitstring at distance $X_t$ for generating the offspring in round $t$. 
\begin{definition}
    We define $\Tilde{h} : [0,\lfloor \frac{n}{2}\rfloor]\rightarrow \R_{\geq 0}$ as 
    \begin{align} \label{drift_h}
        \Tilde{h}(d) = 
            \sum_{r=1}^{n-1} (p_r+p_{n-r}) B(n,d,r),
    \end{align}
    where 
    \begin{align}\label{def:Bndr}
        B(n,d,r)=\sum_{i= \lceil r/2\rceil%\max \{\lceil r/2\rceil, r+d-n\}
        }^{\min\{d, r\}} (2i-r) \frac{{\binom{d}{i}}{\binom{n-d}{r-i}}}{{\binom{n}{r}}}
    \end{align}
    is the drift conditioned on flipping $r$ bits. 
\end{definition}
Note that in the definition above, we assume the convention that an empty sum evaluates to zero. The expression $B(n,d,r)$ was already given in~\cite{DDY2020} as the exact fitness drift with respect to $\OneMax$ when flipping $r$ bits. 

\begin{lemma} \label{lem:drift_is_h}
     If a static unary unbiased algorithm with flip distribution $\calD$ chooses a search point $v_t$ for mutation in step~$t$ such that $d(v_t)=X_t$, (that is, $v_t$ has minimum symmetric distance to the optimum among all search points evaluated so far), then the drift is given by $\Tilde{h}\left(X_t\right)$, i.e.,
    \begin{align}
         \E \left[X_t - X_{t+1} \mid \left\{X_t = d\right\} \land \left\{d\left(v_t\right) = d\right\}\right]=\Tilde{h}(d).
    \end{align}
\end{lemma}
The proof relies on the fact that flipping $n-r$ bits is the same as first flipping $n$ bits and then flipping $r$ bits to adapt the computation of $B(n,d,r)$ given in~\cite{DDY2020}.

\begin{proof}%[Proof of Lemma~\ref{lem:drift_is_h}]
        We assume without loss of generality that $d=\OM (v_t)$ (by the symmetry of $d$, the proof in the case where $d= n- \OM(x)$ is the same after switching the roles of zero-bits and one-bits). This assumption means that $v_t$ is closer to $\vec{0}$ than it is to $\vec{1}$ in terms of Hamming-distance. 
        Consider the event that the algorithm flips $r$ bits in the $t$-th round. We will first calculate the conditional expectation $\Tilde{h}(d)$, when additionally conditioning on this event. \\
        The algorithm makes progress in the following two cases:
        \begin{enumerate}[\bf (a)]
            \item $\OM \left(x^{(t+1)}\right)< d$, or
            \item $n-\OM \left(x^{(t+1)}\right) < d$.
        \end{enumerate}
        \textbf{Case (a).} Let $i$ be the number of bits that are flipped from $1$ to $0$, i.e., the number of bits that make progress towards $\vec{0}$. Then $r-i$ bits are flipped from $0$ to $1$. The probability of this event is 
        \begin{align} \label{hypergeom}
            \Pr [\{\text{flip $i$ one-bits}\}\mid \{\text{flip $r$ bits}\}] = \frac{{\binom{d}{i}}{\binom{n-d}{r-i}}}{{\binom{n}{r}}}. 
        \end{align}
        The progress towards $\vec{0}$ in this event is $i-(r-i)=2i-r$. This is positive if and only if $i> r/2$. Altogether, the following conditions on $i$ must hold to get a positive amount of progress towards $\vec{0}$ in this event: 
        \begin{enumerate}[(i)]
            \item $i> r/2$,
            \item $i\leq r$,
            \item $i\leq d$, and 
            \item $r-i \leq n-d$, or equivalently $i \geq r+d-n$.
        \end{enumerate}
        Note that (i) implies (iv) as $r/2\geq r- n/2 \geq r + d -n$. Conditions (i)--(iii) are equivalent to the limits of the sum in \eqref{def:Bndr}. Hence, the contribution of Case \textbf{(a)} to $\Tilde{h}(d)$ when additionally conditioning on $r$ bits being flipped in round $t$ is $B(n,d,r)$. \\
        % This is positive if and only if $i> r/2$. Considering the conditions that $i\leq d$ and $r-i\leq n-d$, we get from the law of total expectation (conditioning on the number $r$ of bits flipped by the algorithm) for the first case
        % \begin{align} \label{sum1}
        %     \sum_{r=1}^{n-1} p_r B(n,d,r).
        % \end{align}
        % Note that flipping $r=0$ or $r=n$ bits never yields progress with respect to $d$. \\
        \textbf{Case (b).} We observe that flipping $n-r$ bits is equivalent to first  flipping all $n$ bits, and then flipping $r$ bits. Similarly, flipping $r$ bits is the same as first flipping $n$ bits, and then flipping $n-r$ bits.
        The progress that the algorithm makes towards $\vec{0}$ after first flipping all $n$ bits is the same as the progress the algorithm makes towards $\vec{1}$. 
        Overall, we see that the progress towards $\vec{1}$ of flipping $r$ bits is the same as the progress towards $\vec{0}$ of flipping $n-r$ bits. 
        By the same argument as for Case \textbf{(a)}, the contribution of Case \textbf{(b)} to $\Tilde{h}(d)$ when additionally conditioning on $r$ bits being flipped in round $t$ is $B(n,d,n-r)$. 

        To finish the proof, we apply the law of total expectation and get 
        \begin{align}
            \Tilde{h}(d)= \sum_{r=1}^{n-1} p_r (B(n,d,r) + B(n,d,n-r)) =  \sum_{r=1}^{n-1} (p_r+p_{n-r}) B(n,d,r),
        \end{align}
        where the second equality is obtained by splitting into two sums and then reversing the indices.
        We also note that flipping $r=0$ or $r=n$ bits never yields progress with respect to $d$. 
    \end{proof}

Now, we are ready to define the bound $h$ on the drift that we use in our application of the variable drift theorem~\cite[Theorem~9]{DDY2020}.  
% Theorem~\ref{thm:discrete_var_lb}.
\begin{definition}
    Let $h : [0,n]\rightarrow \R_{\geq 0}$,
    \begin{align}
        h(d) = \begin{cases}
            \Tilde{h}(d), \quad & \text{for $d\leq d_0$}\\
            n, \quad & \text{for $d> d_0$},
        \end{cases}
    \end{align}
    where \begin{align}
        d_0 \coloneqq \lfloor (p_1+p_{n-1}) n/\ln^2 n\rfloor.
    \end{align}
\end{definition}

The following statement is adapted from Lemma~21 in~\cite{DDY2020}. 
\begin{lemma}%[adapted from Lemma 21, \cite{DDY2020}]
\label{lem:Bndr_bound}
    There is an $n_0\in \N$ and a constant $C>0$ such that for all $n\geq n_0$, $d\leq d_0$, and $r\geq 2$, it holds
    \begin{align}
        B(n,d,r) < C \cdot (d/n)^2.
    \end{align}
\end{lemma}

\begin{proof}%[Proof of Lemma~\ref{lem:Bndr_bound}] % pick r_0 (possibly in the Lemma).
    Note that the probability in equation \eqref{hypergeom} is the same as the probability that a hypergeometric random variable with parameters $n,r,d$ is equal to $i$. Let $X$ be such a random variable. Since $d \leq d_0$, we have $\E[X ] = \frac{dr}{n}\leq (p_1+ p_{n-1})\frac{r}{\ln^2 n}\leq \frac{r}{\ln^2 n}$.
    
    Let $n_0 = \lceil e^{20}\rceil$. Then $d/n \leq d_0/n \leq 1 / \ln^2 n \leq 1/\ln^2 n_0 \leq 1/400$. Making use of the above observation, the fact that $(2i-r)\leq r$ for all $i\leq r$, and then applying Theorem~\ref{multiplicative_chernoff}, we have
    \begin{align}
    \begin{split}
        B(n,d,r) & = \sum_{i=\lceil r/2\rceil}^{\min \{d,r\}} (2i-r) \Pr [X=i] 
        \leq  r \Pr [X \geq r/2] \label{eq:77} \\
        & = r \Pr \left[X \geq \left( 1 + \left(\frac{n}{2d} -1\right) \right)\cdot \E \left[X\right] \right]
        \\ & \leq r \left(\frac{\exp \left(\frac{n}{2d}-1\right)}{\left(\frac{n}{2d}\right)^{\frac{n}{2d}}}\right)^{\frac{dr}{n}}
         \leq r \frac{e^{r/2}}{\left(\frac{n}{2d}\right)^{r/2}}
        \leq r \left(\frac{2ed}{n}\right)^{r/2} 
        \\ & = 4e^2r \left(\frac{2ed}{n}\right)^{r/2-2} \cdot \left(\frac{d}{n}\right)^2
        % \\ & < 4e^2 r \left(\frac{2e}{\ln n}\right)^{r/2-2} \cdot \left(\frac{d}{n}\right)^2.
         < 4e^2 r \left(\frac{e}{200}\right)^{r/2-2} \cdot \left(\frac{d}{n}\right)^2.
        % \\ & < r \left(2^{-4}\right)^{r/2-3}\cdot \left(d/n\right)^2
        %  = \frac{2^{12} r }{2^{2r}}\cdot \left(d/n\right)^2 \leq \left(d/n\right)^2.
        %2^12 * r < 2^2r
    \end{split}
    \end{align}
    The factor $4e^2r(e/200)^{r/2-2}$ in the upper bound above is decreasing for all $r\geq 1$, so we can plug in $r=1$ to get $C= 4 e^{0.5}200^{1.5}\approx 18653.2$.
    % Noting that $d< n/\ln n$ yields 
    %  \begin{align}
    %     B(n,d,r) & < 4e^2 r \left(\frac{2e}{\ln n}\right)^{r/2-2} \cdot \left(\frac{d}{n}\right)^2.
    %  \end{align}
    % For the last inequality, we used the facts that $r \geq r_0 = 12 $ and $r\leq2^r$.
\end{proof}   

The next lemma gives an upper bound on $h(d)$ that holds once $d$ is small enough. 
\begin{lemma}\label{lem:h_d_p1_On2}
    There is an $n_0 \in \N$ such that for all $n \geq n_0$, and all $d\leq d_0$,%$ = (p_1 + p_{n-1})\frac{n}{\ln^2 n}$,
    \begin{align}
        h(d) \leq \left(1 + \frac{1}{\ln n}\right) \cdot  (p_1 + p_{n-1}) \cdot \frac{d}{n}.
    \end{align}
\end{lemma}
\begin{proof}%[Proof of Lemma~\ref{lem:h_d_p1_On2}]
    Let $C$ be the constant from Lemma~\ref{lem:Bndr_bound}, and let $n_0\coloneqq \lceil e^C \rceil$. As $n_0  \ge  \lceil e^{20} \rceil$, Lemma~\ref{lem:Bndr_bound} is applicable. Note that $B(n,d,1)=d/n$. Together with Lemma~\ref{lem:Bndr_bound}, we have
    \begin{align}
        h(d)= (p_1+ p_{n-1})\frac{d}{n} + \sum_{r=2}^{n-1}  (p_r + p_{n-r})\cdot C
        \left(\frac{d}{n}\right)^2\leq (p_1+ p_{n-1})\frac{d}{n} + C
        \left(\frac{d}{n}\right)^2.
    \end{align}
    By our assumption that $d \leq d_0$, we have \begin{align}
    \begin{split}
        h(d) & \leq (p_1 + p_{n-1 }) \cdot \frac{d}{n} +  \frac{C}{\ln ^2 n} \cdot (p_1 + p_{n-1 })\cdot \frac{d}{n}
        \\ & \leq \left(1 + \frac{1}{\ln n}\right) \cdot  (p_1 + p_{n-1}) \cdot \frac{d}{n},
        \end{split}
    \end{align}
    where the last inequality holds for all $n\geq n_0$.
\end{proof}

As we show in the following lemma, the function $h$ is indeed an upper bound for the change in the potential. We need to show that, under our assumptions, the expected change in the potential conditioning on $d(v_t)=d + \Delta$ is maximal if $\Delta = 0$.
The proof relies on a case distinction to deal with different ranges of $d$ and~$\Delta$. Depending on the case, we use an additive Chernoff bound, a multiplicative Chernoff bound, or Lemma~\ref{lem:h_d_p1_On2}.

\begin{lemma}\label{lem:h_bounds_drift}
    There is an $n_0\in \N$ such that for all $n\geq n_0$, any static unary unbiased algorithm with flip distribution $\calD$ such that $p_1+p_{n-1}= n^{-o(1)}$, and all $d\leq d_0$, it holds
    \begin{align}
        \E \left[X_t - X_{t+1} \mid X_t = d \right] \leq h(d).
    \end{align}
\end{lemma}
\begin{proof}%[Proof of Lemma~\ref{lem:h_bounds_drift}]
    %Denote by $v_t\in (x_0, \dots, x_t)$ the search point that the algorithm uses for generating the offspring in round $t$. 
    We start by defining some notation. Let $d'\geq d$ and put
    \begin{align}
        h_d(d')\coloneqq \E \left[X_t - X_{t+1} \mid \{X_t = d\} \land \{d(v_t) = d'\}\right].
    \end{align} 
    Note that \begin{align} \label{max_eq}
        \E \left[X_t - X_{t+1} \mid X_t = d \right] \leq \max_{d \leq d' \leq n/2}  
        h_d(d'). 
    \end{align}
    If $d(v_t)= d$, it follows from Lemma \ref{lem:drift_is_h} that the drift is exactly $h(d)$, i.e., we have $h_d(d)=h(d)$. So it remains to show that the above maximum is attained at $d'=d$. We denote by $\Delta \coloneqq d'-d.$ We need to show for all $n/2-d \geq \Delta\geq 1$ that \begin{align}
        h_d(d+\Delta)\leq h(d).
    \end{align}
    We have \begin{align}
        h(d ) \geq (p_1 + p_{n-1}) \cdot B(n,d,1) = (p_1 + p_{n-1}) \frac{d}{n } \geq \frac{d}{n^{1+o(1)}}.
    \end{align}
    We observe that to make progress beyond the best-so-far search point, the algorithm needs to flip at least $\Delta + 1$, and at most $n-\Delta -1$ bits.
    Together with a similar calculation as in the proof of Lemma \ref{lem:drift_is_h}, this yields
    \begin{align}
        h_d(d+\Delta)= \sum_{r=\Delta+1}^{n-\Delta -1 } (p_r+p_{n-r})B_d(n,d+ \Delta,r), \label{eq:68}
    \end{align}
    where 
    \begin{align}
        B_d(n,d + \Delta, r )= \sum_{i= %\max \{
        \lceil \frac{r + \Delta}{2}\rceil 
        %,r + d + \Delta -n\} 
        }^{\min \{d+ \Delta, r\}}
        (2i-r-\Delta)\frac{\binom{d+\Delta}{i}\binom{n-d-\Delta}{r-i}}{\binom{n}{r}}. \label{eq:hyperhyper}
    \end{align}
    We remark that $B_d(n, d+\Delta , r)\leq B(n, d+\Delta, r)$: The range of the sum in \eqref{eq:hyperhyper} is contained in the range of the sum in \eqref{def:Bndr} (plugging in $d+\Delta$). Consider an index $i$ appearing in both sums. The fraction in the $i$-th summand is the same, so the claimed inequality follows from $2i-r-\Delta \leq 2i-r$.
    We note that the range of the sum in \eqref{eq:68} is contained in the range of the sum in \eqref{drift_h}, so with the previous remark, we have $h_d(d+\Delta)\leq h(d+\Delta)$.
    
    As in the proof of Lemma \ref{lem:Bndr_bound}, the fraction in equation \eqref{eq:hyperhyper} is equal to the probability of a hypergeometric random variable $X$ with parameters $n, r, d+ \Delta$ being equal to $i$. Let $X$ be such a random variable. It holds $\E [X] = \frac{r(d+\Delta)}{n}\leq \frac{r}{2}$. 

    In the following, we distinguish four cases \textbf{(I)}, \textbf{(II)}, \textbf{(IIIa)}, and \textbf{(IIIb)}, specifying different ranges of $\Delta$ and $d$.

    \textbf{Case (I): $\Delta \geq n/12$.} Using the same argument as for equation \eqref{eq:77}, and then applying Theorem~\ref{additive_chernoff}, we get
    \begin{align}
    \begin{split}
        B_d(n,d+\Delta,r)& = \sum_{i= \lceil \frac{r+\Delta}{2}\rceil}^{\min\{d+\Delta, r\} } (2i-r-\Delta) \Pr \left[X =i \right]\leq r 
        \Pr \left[X \geq \frac{r+ \Delta}{2} \right]
        \\ & \leq r \Pr \left[X \geq \E [X]+ \frac{\Delta}{2} \right]
         \leq r \exp{ \left(-\frac{\Delta^2}{2n}\right)} 
         \leq n \exp \left( - \frac{n}{288} \right).
    \end{split}
    \end{align}
    We have \begin{align}
        h_d(d+\Delta) \leq n^2 e^{-\Omega(n)}=o\left(\frac{d}{n^{1 + o(1)}}\right)=o(h(d)). %o(d/n^{1+o(1)}).
    \end{align}

    \textbf{Case (II): $n/12 > \Delta \geq \ln^2 n$.} We proceed as in the proof of Lemma \ref{lem:Bndr_bound}, applying Theorem~\ref{multiplicative_chernoff}. Note that, in contrast to the proof for case (1), we are bounding $B$. Choosing $n_0$ large enough such that $d_0 < n/12$, we have
    \begin{align}
    \begin{split}
        B(n,d+\Delta,r) & \leq r \Pr \left[X \geq \frac{r}{2} \right] = r \Pr \left[X \geq \frac{n}{2(d+ \Delta)}\E [X] \right]
        \\ & \leq r \left(\frac{2(d+\Delta)e}{n}\right)^{r/2} \label{eq:b_dndr_bound}
        \leq r \left(\frac{e}{3}\right)^{r/2}.
    \end{split}
    \end{align}
    For the last inequality, we used that $d\leq d_0 < n/12$ and our assumption $\Delta \leq n/12$. In particular, we have for $r > \Delta \geq \ln^2 n $
    \begin{align}
        B(n,d+\Delta, r) = O(n^{-\Omega (\ln n)}).
    \end{align} 
    We have \begin{align}
        h_d(d+\Delta) \leq n O(n^{-\Omega (\ln n)}) = o(h(d)).%o(d/n^{1+o(1)}).
    \end{align}

    \textbf{Case (IIIa): $1 \leq \Delta < \ln^2 n$ and $d< n^{1/4}-\ln^2 n $.} By the same argument as for \eqref{eq:b_dndr_bound}, we have
    \begin{align}\label{eq:75}
        B(n,d+\Delta, r ) \leq r \left(\frac{2(d+\Delta)e}{n}\right)^{r/2} < r\left(\frac{2e}{n^{3/4}}\right)^{r/2}< n\left(\frac{2e}{n^{3/4}}\right)^{r/2}.
    \end{align}
    For $r\geq 9$, we thus get $B(n,d+\Delta,r) \leq \Theta(n^{-2.375})$. For $3\leq r \leq 8$, we get $B(n,d+\Delta,r) \leq  \Theta(n^{-9/8})$ (note that we used the second to last bound in \eqref{eq:75} for this case). Furthermore, we have for $r=2$, where only $\Delta = 1$ is relevant (otherwise it does not appear in~\eqref{eq:68}),
    \begin{align}
        B_d(n, d + 1,2 ) = \frac{\binom{d+1}{2}}{\binom{n}{2}} = \frac{(d+1)\cdot d}{n\cdot (n-1)} \leq \frac{n^{2/4}}{(n-1)^2}\leq \Theta(n^{-3/2}).
    \end{align}
    We compute
    \begin{align}
        h_d(d+\Delta) &  \leq \Theta(n^{-3/2})+ 6 \cdot \Theta(n^{-9/8}) + n \Theta(n^{-2.375}) 
        = \Theta (n^{-9/8}) = o(h(d)).%o(d/n^{1+o(1)}).
    \end{align}

    \textbf{Case (IIIb): $1 \leq \Delta < \ln^2 n$ and $d_0\geq d\geq n^{1/4}-\ln^2 n$.}
    Note that \begin{align}
        h_d(d+\Delta) \leq h(d + \Delta)- (p_1+p_{n-1})B(n,d+\Delta,1),
    \end{align}
    since $B_d(n,d+\Delta,r)\leq B(n,d+\Delta,r)$, and $r=1$ does not appear in the sum \eqref{eq:68} defining $h_d(d + \Delta)$, as $\Delta\geq 1$.
    
    We choose $n_0$ large enough such that $n_0^{1/4}\ln n_0 -n_0^{1/4} - \ln^3 n_0> 0$. This implies $d(\ln n-1 ) > \ln^2 n$. Together with Lemma \ref{lem:h_d_p1_On2}, we have
    \begin{align}
    \begin{split}
        h_d(d+\Delta) & \leq h(d+\Delta)-(p_1+p_{n-1})\frac{d+\Delta}{n} 
         \leq \frac{1}{\ln n } \cdot (p_1+p_{n-1})\frac{d+\Delta}{n}
        \\ &  \leq  \frac{1}{\ln n } \cdot (p_1+p_{n-1})\frac{d+\ln^2 n}{n}
          <  \frac{1}{\ln n } \cdot (p_1+p_{n-1})\frac{d+d(\ln n-1)}{n} \\ & \leq h(d),
    \end{split}
    \end{align}
    concluding the proof of this lemma.
\end{proof}

Finally, we investigate the monotonicity of the function $h$. 

\begin{lemma}\label{lem:monotonic}
    There is an $n_0 \in \N$ such that for all $n \geq n_0$, the function $h$ is monotonically increasing.
\end{lemma}
\begin{proof}
    Let $1\leq d \leq d_0-1$. Take $n_0 = e^2$. Then we have $d/n\leq d_0/n \leq 1/\ln^2 n \leq 1/4$.
    We will show that $h(d)\leq h(d+1)$. Recall that 
    \begin{align}
        h(d) = \sum_{r=1}^{n-1} (p_r+p_{n-r}) B(n,d,r).
    \end{align}
    Let $1\leq r \leq n-1$. We will show that 
    \begin{align}
        B(n,d,r) \leq B(n,d+1,r).
    \end{align}
    Recall the definition
    \begin{align}
        B(n,d,r)=\sum_{i=\lceil r/2\rceil}^{\min\{d, r\}} (2i-r) \frac{\binom{d}{i}\binom{n-d}{r-i}}{\binom{n}{r}}.
    \end{align}
    %The range of this sum for $B(n,d,r)$ is contained in the range of this sum for $B(n,d+1,r)$, as otherwise, we would have $r+d+1-n > \lceil r/2\rceil\geq r/2$. However, $r+d+1-n \leq r+n/4-n <r-3r/4= r/4$, which yields a contradiction.
    The range of the sum for $B(n,d,r)$ is contained in the range of this sum for $B(n,d+1,r)$. Thus, it is enough to show that for all $\lceil r/2\rceil\leq i\leq \min\{d, r\}$, it holds \begin{align}
        (2i-r) \frac{\binom{d}{i}\binom{n-d}{r-i}}{\binom{n}{r}}\leq (2i-r) \frac{\binom{d+1}{i}\binom{n-d-1}{r-i}}{\binom{n}{r}},
    \end{align}
    or equivalently, 
    \begin{align} \label{eq:factorial_fracs}
        &\frac{d!}{i!(d-i)!} \cdot \frac{(n-d)!}{(r-i)!(n-d-r+i)!} \leq 
        \frac{(d+1)!}{i!(d+1-i)!} \cdot \frac{(n-d-1)!}{(r-i)!(n-d-1-r+i)!}
    \end{align}
    Equation \eqref{eq:factorial_fracs} can be rewritten as \begin{align}
        \frac{n-d}{n-d-r+i}\leq \frac{d+1}{d+1-i}.
    \end{align}
    Multiplying by the denominators (which are always positive) and expanding the product, this is equivalent to 
    \begin{align} % what about i = r + d - n? -- We have already shown that this does not occur.
        dn -d^2+n -d -in+di\leq dn + n -d^2-d-dr-r+di+i.
    \end{align}
    Finally, we see that this is the same as 
    \begin{align}
        dr+r\leq in+ i,
    \end{align}
    or equivalently
    \begin{align}
        r(d+1)\leq i(n+1).
    \end{align}
    This inequality is indeed true, as \begin{align}
        i(n+1)\geq \frac{r}{2} (n+1) = r \Big(\frac{n}{2}+\frac{1}{2}\Big)\geq r \left(\frac{n}{4}+1\right)\geq r(d+1),
    \end{align}
    concluding the proof of this lemma, as $h(d_0)\leq n = h(d_0+1)$, and $h$ is equal to $n$ on $[d_0+1,n]$.
\end{proof}

With this lemma at hand and Lemma~13 from~\cite{DDY2020}, which bounds the probability to make large jumps, we can then show Theorem~\ref{thm:lwr_bnd}, using very similar computations as those that were used in~\cite{DDY2020}. 

\begin{lemma}[Lemma 13, \cite{DDY2020}] \label{lem:c_tilde}
    There exists an $n_0\in \N$ such that for all $n\geq n_0$, for all $r \in [0,n]$, and for all $x\in \{0,1\}^n $, it holds that 
    \begin{align}
        \Pr \left[d(\text{\emph{flip}}_r(x))\geq \Tilde{c} (d(x))\right]\geq 1- n^{-4/3} \ln^7 n,
    \end{align}
    where 
    \begin{align}\label{eq:definition-ctilde}
        \Tilde{c} : [n]\rightarrow [0,n], i \mapsto \Tilde{c}(i)\coloneqq 
        \begin{cases}
        i- \sqrt{n}\ln n, \quad &\text{\emph{if }} i \geq n/6 \\
        i-\ln^2n, \quad &\text{\emph{if }} n^{1/3}\leq i < n/6\\
        i-1, \quad &\text{\emph{if }} i < n^{1/3}.
        \end{cases}
    \end{align}
\end{lemma}
To interpret the lemma above, assume that $y$ is the result of flipping $r$ bits in a search point $x$. Then $\Tilde{c}$ is a function such that the probability of $d(y)<\Tilde{c}(d(x))$, i.e., the probability that $y$ made a large gain compared to $x$ w.r.t. the distance function $d$, is very small, namely at most $n^{-4/3} \ln^7 n$.

\begin{theorem} \label{thm:lwr_bnd}% define static unary unbiased 
    The expected runtime of any static unary unbiased algorithm with flip distribution $\calD$ satisfying $p_1+p_{n-1}= n^{-o(1)}$ on any function $f: \{0,1\}^n\rightarrow \R$ with unique global optimum is at least 
    \begin{align}
        %\sum_{d=1}^{n/4} \frac{1}{h(d)} - \Theta\left(n^{2/3+ o(1)} \ln^9 n\right).
        \sum_{d=1}^{d_0} \frac{1}{h(d)} - o\left(n\right).
    \end{align}
\end{theorem}
\begin{proof}[Proof of Theorem~\ref{thm:lwr_bnd}]
    Assume that $n$ is large enough so that Lemma~\ref{lem:h_bounds_drift}, Lemma~\ref{lem:monotonic}, and Lemma~\ref{lem:c_tilde} apply. Note that the time $T_A$ that a unary unbiased algorithm takes to optimize $f$ is at least as large as the time $T$ when the potential $X_t$ reaches zero for the first time. 
%the hitting time $T$ of $0$ of the potential $X_t$.
    Every %unary unbiased 
    algorithm 
    $A$ has to generate its inital search point $x^{(0)}$ uniformly at random.
    We have $\E [\OM(x^{(0)})]= n/2$. By the bounds of Theorem~\ref{additive_chernoff}, $\Pr[\OM(x^{(0)})<\frac{n}{4}]\leq e^{-n/8}$ and $\Pr[\OM(x^{(0)})>\frac{3n}{4}]\leq e^{-n/8}$. Hence, $\Pr[X_0<\frac{n}{4}]\leq 2e^{-n/8}$.
    For now, we will assume that $X_0\geq n/4$. 
    
    Let $c: [n]\rightarrow [0,n], i \mapsto c(i) \coloneqq \min \{\Tilde{c}(j)\mid j\geq i \}$. It holds $c(i)\leq \Tilde{c}(i)$ for all $i\in [n]$ by definition. Furthermore, $c$ is monotonically increasing. By Lemma \ref{lem:c_tilde}, for all $r \in [0,n]$, $x\in\{0,1\}^n$, and $d'\leq d(x)$,
    \begin{align}
        \Pr \left[d(\text{flip}_r(x))\geq c(d')\right]& \geq \Pr \left[d(\text{flip}_r(x))  \geq \Tilde{c}(d')\right] \geq 1- n^{-4/3} \ln^7 n.
    \end{align} 
    By the law of total probability, it follows \begin{align}
        \Pr \left[ X_{t+1}\geq c(X_t)\right] & \geq \sum_{r=0}^n p_r \Pr \left[d(\text{flip}_r(v_t))\geq c(d(v_t))\right]
        \geq  1- \underbrace{n^{-4/3} \ln^7 n}_{\eqqcolon p}.
    \end{align}
    Together with Lemma \ref{lem:h_bounds_drift}, we see that all conditions of Theorem~\ref{thm:discrete_var_lb} are satisfied for $(X_t)_{t\geq 0}$, $h$, $c$, and $p$. 
    Recall that $\mu(x) \coloneqq \max \{i \mid c(i) \leq x\}$, so by definition $\mu(x)= \max \{i \mid \min \{\Tilde{c}(j )\mid j \geq i\} \leq x \}$. We can simplify this via $\max \{i \mid \min \{\Tilde{c}(j )\mid j \geq i\} \leq x \} = \max\{i \mid \Tilde{c}(i) \leq x\}$. (Assume for contradiction that the maximum is attained at $i$, and the inner minimum is attained at some $j>i$. But then $i$ is not maximal.) We plug in the definition of $\tilde c$ from~\eqref{eq:definition-ctilde} and obtain 
    \begin{align}
        \mu (i) = \max\{i \mid \Tilde{c}(i) \leq x\} = \begin{cases}
            i+ 1  & \text{for $i< n^{1/3}- \ln^2 n$,}\\
            i + \ln^2 n & \text{for $n^{1/3}- \ln^2 n\leq i < n/6-\sqrt{n} \ln n$,}\\ 
            i + \sqrt{n} \ln n & \text{for $n/6-\sqrt{n} \ln n\leq i< n/2- \sqrt{n} \ln n$,} \\
            \lfloor n/2 \rfloor & \text{for $n/2- \sqrt{n} \ln n \leq i< n/2$.}
        \end{cases}
    \end{align}
    Recall that $g(x)=\sum_{i=0}^{x-1} 1/h(\mu(i))$. By Theorem~\ref{thm:discrete_var_lb}, we have \begin{align}
        \E \left[T \mid X_0\right] \geq g(X_0) - \frac{g^2 (X_0)p}{1 + g(X_0 )p }.
    \end{align}
    First, we will bound $\frac{g^2 (X_0)p}{1 + g(X_0 )p}$. We have $h(d)\geq (p_1+ p_{n-1}) \cdot B(n,d,1) =(p_1+ p_{n-1}) \cdot \frac{d}{n}$. As $h $ is monotonically increasing (Lemma~\ref{lem:monotonic}), and $\mu(i)\geq i$, we have $h(x)\leq h(\mu (x))$. Hence, \begin{align}% here we start from 1, but def. of g starts from 0.
        g(X_0)& \leq \sum_{i=1}^{x-1} \frac{1}{h(x)} \leq \frac{1}{(p_1+ p_{n-1})} \sum_{i=1}^{x-1} \frac{n}{x }  = \frac{1}{(p_1+ p_{n-1})} O(n \ln n) = n^{1+o(1)}.\label{eq:bndsum}
    \end{align}
    It follows that \begin{align}
        \frac{g^2 (X_0)p}{1 + g(X_0 )p}\leq g^2(x)p = n^{2+o(1)-4/3}\ln^7 (n) = o(n).%O\left(n^{2/3 + o(1)}\ln^9 n\right).
    \end{align}
    Next, we bound $g(X_0 )$ from below. As $h$ is monotonic (Lemma~\ref{lem:monotonic}), and all summands are positive,
    \begin{align}
    \begin{split}
        g(X_0 ) & = \sum_{x=0}^{X_0-1 } \frac{1}{h(\mu(x))} 
         \geq \sum_{i=1}^{n^{1/3}-\ln^2 n} \frac{1}{h(x)}+ \sum_{x = n^{1/3}}^{n/6 - \sqrt{n} \ln n + \ln^2 n } \frac{1}{h(x)}
        + \sum_{x=n/6}^{X_0} \frac{1}{h(x)} 
        \\ & \geq \sum_{x=1}^{X_0} \frac{1}{h(x)} - \frac{\ln^2 n}{h(n^{1/3}-\ln^2 n)} - \frac{\sqrt{n} \ln n - \ln^2 n}{h(n/6 - \sqrt{n} \ln n + \ln^2 n   )}. 
    \end{split}
    \end{align}
    We apply $h(x ) \geq(p_1+ p_{n-1})  \frac{d}{n}$ again and get 
    \begin{align}
    \begin{split}
        g(X_0 ) & \geq \sum_{x=1}^{X_0} \frac{1}{h(x)} - \frac{1}{p_1+ p_{n-1}} \cdot \left(\frac{n\ln^2 n}{n^{1/3}-\ln^2 n} + \frac{n (\sqrt{n} \ln n - \ln^2 n)}{n/6 - \sqrt{n} \ln n + \ln^2 n} \right)
       \\  & \geq \sum_{x=1}^{X_0} \frac{1}{h(x)} -  \frac{1}{p_1 + p_{n-1} } \cdot \Theta \left( n^{2/3} \ln^2 n \right) 
        \geq \sum_{x=1}^{X_0} \frac{1}{h(x)} - o(n).% \Theta \left( n^{2/3+o(1)} \ln^2 n \right).
    \end{split}
    \end{align} 
    Finally, we can conclude using the law of total expectation 
    \begin{align}
        \E \left[T_A \mid X_0\right] & \geq \left(1-2e^{-n/8}\right)\cdot \Bigg(\sum_{x=1}^{n/4} \frac{1}{h(x)}  - o(n)- o(n)%\Theta \left( n^{2/3+o(1)} \ln^2 n \right) 
        %\\ & \quad  -O\left(n^{2/3 + o(1)}\ln^9 n\right)
        \Bigg) % Put this under other minus.
         \geq \sum_{x=1}^{d_0} \frac{1}{h(x)} -o(n), % \Theta \left( n^{2/3+o(1)} \ln^9 n \right),
    \end{align}
    where we used that $n/4 \geq d_0$ for $n$ large enough, and note that by \eqref{eq:bndsum}, the sum multiplied by $2e^{-n/8}$ is $o(n)$.
\end{proof}

With this statement at hand, we can finally prove Theorem~\ref{thm:lower-bound}. 
\begin{proof}[Proof of Theorem~\ref{thm:lower-bound}] 
    By Lemma \ref{lem:h_d_p1_On2}, we have for all $1\leq d \leq d_0 =n \frac{p_1+ p_{n-1}}{\ln^2 n}$, \begin{align}
        \frac{1}{h(d)}& \geq \frac{n}{(p_1 + p_{n-1})d}- \frac{\frac{1}{\ln n}}{1 + \frac{1}{\ln n}}\frac{n}{(p_1 + p_{n-1})d}
         = (1-o(1))\frac{n}{(p_1 + p_{n-1})d}.
    \end{align}
    Applying Theorem~\ref{thm:lwr_bnd} yields
    \begin{align}\label{eq:thm:lower-bound-summary}
        \E [T] & \geq \left((1- o(1)) \frac{n}{(p_1 + p_{n-1})} \sum_{d=1}^{d_0} \frac{1}{d}\right) - o(n) 
         % \\ &   = (1 \pm o(1)) \frac{n}{(p_1 + p_{n-1})} \ln\left(\frac{n(p_1 + p_{n-1}) }{\ln n} \right)
         % \\ & \geq (1 \pm o(1)) \frac{n}{(p_1 + p_{n-1})} \left(\ln (n^{1+o(1)}) - \ln \ln n\right)
           = (1 \pm o(1))  \frac{1}{(p_1 + p_{n-1})} n \ln n,
    \end{align}
    where we used that $\ln d_0 \sim \ln n.$
\end{proof}

\section{Discussion of Upper and Lower Bounds}
\label{sec:discussion}

We now discuss that some requirements on $p_1$ (Section~\ref{sec:upper-tight}) and $\chi$ (Section~\ref{sec:chi}) are necessary, that $p_{n-1}$ cannot be dropped from Theorem~\ref{thm:lower-bound} (Section~\ref{sec:lower_pn1}) and that, for general unbiased mutation operators, being closer to the optimum does not necessarily imply a larger probability of sampling the optimum and that \OneMax is not the easiest function for all \ooeaD instances (Section~\ref{sec:domination}).

\subsection{Requirement on \texorpdfstring{$p_1$}{p1}}
\label{sec:upper-tight} 

% \textbf{Requirement on $p_1$.} 
We start with a proposition saying that the leading constant can be smaller if $p_1 = n^{-c}$ for any $c>0$. In fact, this is already the case for \onemax, as the following example shows. 
\begin{proposition}\label{prop:small-p1}
Let $0<c<1$ be constant. Consider the \ooeaD with distribution $\calD = (p_0,p_1,\ldots,p_n)$ defined by $p_1 = n^{-c}$ and $p_2 = 1-p_1$. Then there exists $\eps >0$ such that for sufficiently large $n$ the expected runtime on \onemax is at most
\begin{align*}
(1-\eps)\cdot \frac{1}{p_1}\cdot n\ln n.
\end{align*}
\end{proposition} 

\begin{proof}%[Proof of Proposition~\ref{prop:small-p1}]
Let $d_0 \coloneqq n^{1-c/2}$. We consider two phases of the algorithm: one for reducing the distance from the optimum from at most $n$ to $d_0 +1$, and the other for reducing it from $d_0+1$ to $0$. Let $T_1$ and $T_2$ be the respective time spent in those two phases. 

To bound $T_1$, assume the current distance is $d \geq d_0+1$. Then the probability that both bits of a $2$-bit flip are zero-bits is $\binom{d}{2}/\binom{n}{2} \ge (d_0^2/2)/(n^2/2) = n^{-c}$. We need at most $(n-d_0)/2$ such flips to leave the first phase, and thus
\begin{align}
    \E[T_1] \le \frac{n^c}{p_2}\cdot\frac{n-d_0}{2} \le \frac{1}{p_2}\cdot\frac{n^{1+c}}{2} = \frac{1}{p_2}\cdot \frac{n}{2p_1} = o\left(\frac{1}{p_1}n\ln n\right)
\end{align} 
rounds. Hence, the first phase is asymptotically faster than the claimed runtime bound. 

For the second phase, if $X_t$ is the distance from the optimum in round $t$, then every single-bit flip has a chance of $X_t/n$ to reduce $X_t$ by one. Since single-bit flips occur with probability $p_1$, we have
\begin{align}
    \E[X_{t}-X_{t+1} \mid X_t = d] \ge p_1 \cdot \frac{d}{n}.
\end{align}
The second phase starts with $X_0 \le d_0 + 1$, so by the multiplicative drift theorem the duration of the second phase is at most
\begin{align}
    \E[T_2] \le \frac{1+\ln(d_0+1)}{p_1/n} = \frac{n+(1-c/2)n\ln n + o(n)}{p_1},
\end{align}
where we used that $\ln(d_0 + 1) = \ln(d_0) + \ln(1 + 1/d_0) = \ln(d_0) + o(1)=(1-c/2)\ln n+ o(1)$. We obtain
\begin{align}
    \E[T_1+T_2] \le (1+o(1))\cdot \frac{1-c/2}{p_1}\cdot n\ln n,
\end{align}
and the proposition follows with $\eps \coloneqq c/3$.
\end{proof}

The point of Proposition~\ref{prop:small-p1} is that we have a strictly smaller leading constant than in  %Theorem~\ref{thm:upperOM} and 
Corollary~\ref{cor:upper-bound} and in Theorem~\ref{thm:main-short}. The reason for this effect is that if $p_1 = n^{-\Omega(1)}$, for Hamming distances in the range $[n^{1-c/2},n]$ from the optimum, two-bit flips are more effective than single-bit flips. This range is thus traversed more quickly. With single-bit flips, the algorithm would need time $\Omega(n\ln n/p_1)$ to traverse this region, but it can be traversed in time $o(n\ln n/p_1)$ by two-bit flips. Hence, the time spent in this phase becomes negligible. Even though this region is still far away from the optimum, it consumes a constant fraction of the total expected runtime if the algorithm is restricted to single-bit flips. Hence, the speed-up from two-bit flips eliminates this constant fraction from the total expected runtime, and thus reduces the leading constant of the total expected runtime. This shows that the lower bound in Theorem~\ref{thm:main-short} cannot hold if $p_1 = n^{-\Omega(1)}$. On the other hand, we will show in Theorem~\ref{thm:linear-lower} below that it does hold for all $p_1=n^{-o(1)}$, which is tight by the above discussion.

\subsection{Requirement on \texorpdfstring{$\chi$}{chi}}
\label{sec:chi} 
% \textbf{Requirement on $\chi$.} 
In contrast to the requirement on $p_1$, we could not prove that the runtime changes when the requirement on $\chi$ is not satisfied. However, Proposition~\ref{prop:large-mu-general} below shows that, close to the optimum, the behavior of the algorithm changes substantially if $\chi$ is large. Recall that from any parent at distance one from the optimum, we have a probability of $p_1\cdot 1/n$ to create the optimum as offspring. Hence, one might naively expect to wait at most for $n/p_1$ rounds in expectation to find the optimum. However, this is wrong for large values of $\chi$, as we show in Proposition~\ref{prop:large-mu-general}. Before we show this, we first provide a lower bound on the number of offspring that any unary unbiased algorithm needs to generate to reach the optimum when starting at distance $d$ from it, regardless of $\chi$.

\begin{lemma}\label{lem:large-mu}
    Any unary unbiased algorithm that starts in a search point in distance $1\le d \le n/2$ of $\vec 1$ needs to generate at least $\Omega(n\ln^+ d)$ offspring in expectation before generating $\vec 1$. If $d=\omega(1)$, the algorithm needs to generate $\omega(n)$ offspring \emph{with high probability}.
\end{lemma}
\begin{proof}
For $d=\omega(\ln \ln n)$, this was proven in~\cite{lehrewitt2012}.\footnote{The result was only formulated explicitly for a random starting point in~\cite[Theorem~6]{lehrewitt2012}. However, the proof was based on~\cite[Theorem~5]{lehrewitt2012}, which holds for arbitrary values of $d$. The proof chooses a suitable $s_{\min} = \Theta(\ln \ln n)$ and showed that at least $\Omega(n(\ln d - \ln{s_{\min}}))$ steps are needed, which implies our claim for $d = \omega(\ln \ln n)$. The proof was formulated in the context of the \onemax function, but it does not make use of the underlying function and thus holds for any fitness function.} For smaller $d = O(\ln \ln n)$, let $X_t$ be the \onemax value of the $t$-th generated offspring from $x$, and let $Z_t \coloneqq\min_{\tau \le t} \min\{X_\tau, n-X_\tau\}$ be the best-so-far distance from either $\vec 1$ or its complement $\vec 0$. Consider a $k$-bit flip of any parent $y$ with any $2\le k \le n-2$, and consider the probability $p_{\text{imp}}$ to create an offspring $z$ with $\OM(z) < \min\{\ln n, \OM(y), n-\OM(y)\}$, i.e., the probability of improving $Z_t$ by creating an offspring close to $\vec 0$. The other case of creating an offspring close to $\vec 1$ is symmetric. If $\OM(y) \ge k+\ln n$ then we have $p_{\text{imp}}=0$, so assume otherwise. 

We first consider the case $k \le n^{1/6}$ and hence $\OM(y) < n^{1/6} + \ln n$ and claim that then $p_{\text{imp}} = O(k^2\OM(y)^2/n^2) = O(n^{-4/3})$. Since $k\ge 2$, for decreasing $\OM(y)$ it is necessary to flip at least two one-bits. Let us uncover the $k$ bits to be flipped one by one, i.e., the first bit is selected uniformly at random from all $n$ positions, the second bit from the remaining $n-1$ positions and so on. Note that even when some of the $k$ bits have been selected, there are still at least $n-k+1 \ge n/2$ zero-bits remaining that could be selected, and most $\OM(y)$ one-bits. Hence, in each step the probability of selecting a one-bit is at most $2\OM(y)/n$. By a union bound, the probability of selecting at least one one-bit among the $k$ steps is at most $2k\OM(y)/n$. If that happens after $i$ steps, for any $1\le i\le k$, the probability to select another one-bit with the remaining $k-i$ steps is by the same argument again at most $2k\OM(y)/n$. Hence, the probability of flipping at least two one-bits is at most $4k^2\OM(y)^2/n^2 = O(n^{-4/3})$.

For the other case $k > n^{1/6}$, note that if $\ell$ zero-bits are flipped then the offspring $z$ has $\OM(z) \ge \ell$, no matter what happens to the other bits. Hence, to obtain a $\OM(z) < \ln n$, it is necessary to flip at most $\ln n$ zero-bits. Since the expected number of flipped zero-bits is $\Omega(k) = \Omega(k^{1/6})$, by the Chernoff bound this event is very unlikely, $p_{\text{imp}} = e^{-\Omega(n^{1/6})}$. So in all cases, $p_{\text{imp}} = o(1/(n\ln n))$. Now consider a sequence of $n\ln n$ steps of the algorithm. With high probability, none of the steps with $2\le k \le n-2$ reduces $Z_t$ by a union bound. Then $Z_t$ is only reduced in rounds with $k=1$ or $k=n-1$ bit-flips, and of those $\Omega(n\ln^+ d)$ are needed, as the process then becomes the coupon collector process~\cite{doerr2020probabilistic}. Note that this also shows that the process needs $\omega(n)$ steps \emph{with high probability} if $d=\omega(1)$ since the same is true for the coupon collector process.
\end{proof}

\begin{proposition}\label{prop:large-mu-general}
Let $x$ be a search point at Hamming distance one from the optimum $\vec 1$. Let $\calD$ be any probability distribution on $[0,n]$ with mean $\chi$. For a linear function $f$, let $T^{x}_{\calD}(f)$ be the number of iterations until the \ooeaD with starting position $x$ finds the optimum.
\begin{enumerate}[(a)]
\item There is a linear function $f$ depending on $x$ such that $\E[T^{x}_{\calD}(f)] = \Omega(n\ln^+ \chi)$.
\item If $\chi = \omega(1)$ then there is a linear function $f$ depending on $x$ such that $T^{x}_{\calD}(f) = \omega(n)$ with high probability.
\item If $\chi = O(1)$ and $p_1 = \Theta(1)$ then $\E[T^{x}_{\calD}(f)] = O(n)$ for every linear function~$f$.
\end{enumerate}
\end{proposition}

\begin{proof}%[Proof of Proposition~\ref{prop:large-mu-general}]
%Before we start with the proof, we first argue that any unary unbiased algorithm that starts in some search point in distance $1\le d \le n/2$ of $\vec 1$ needs to generate at least $\Omega(n\ln^+ d)$ offspring in expectation before generating $\vec 1$. 
%
\emph{(a) and (b).} The statement in (a) for $\chi = O(1)$ follows immediately from Lemma~\ref{lem:large-mu} since the algorithm needs at least $\Omega(n)$ steps. In particular, we may assume $\chi \ge 4$ from now on.
%For every $k \in \{1,\ldots,n-1\}$, and every search points $x$, the number of search point in Hamming distance exactly $k$ from $x$ is at least $n$. Since a $k$-bit flip of $x$ samples from those points uniformly at random, the probability to sample the optimum or its complement is at most $2/n$. Hence the expected number of iterations to find the optimum or its complement is at least $n/2$, which shows (a) for $\chi = O(1)$. In particular, we may assume $\chi \ge 4$ from now on.

Since $x$ has Hamming distance one from $\vec 1$, it differs in exactly one position from $\vec 1$. Without loss of generality, we assume that this is the first position. Then we define $f=f_x$ via $f_x(y) \coloneqq n\cdot y_1 + \sum_{i=2}^n y_i$, i.e, we give weight $n$ to the first position and weight $1$ to all other positions. When in $x$, the algorithm will accept any offspring that flips the first position. Let us call $R$ the number of bits that are flipped in this mutation. Then we may compute the distribution of $R$ via Bayes formula as
\begin{align}\label{eq:Bayes}
\Pr[R=r] = \frac{\Pr[R=r \text{ and pos. 1 flipped}]}{\sum_{s\in[n]} \Pr[R=s \text{ and pos. 1 flipped}]}.
\end{align}
Note that $\Pr[R=s \text{ and pos. 1 flipped}] = \Pr[R=s] \cdot \Pr[\text{pos 1 flipped} \mid R=s] = p_s \cdot s/n$, where the conditional probability is $s/n$ since the mutation operator is unbiased. Hence,~\eqref{eq:Bayes} simplifies to
\begin{align}\label{eq:Bayes2}
\Pr[R=r] = \frac{p_r \cdot r/n}{\sum_{s\in[n]} p_s\cdot s/n} = \frac{p_r \cdot r}{\chi}.
\end{align}
In particular, this implies for every $\gamma \le 1$,
\begin{align*}\label{eq:Bayes}
\Pr[R\le \gamma \chi] & = \sum_{r=1}^{\lfloor\gamma \chi\rfloor} \frac{p_r \cdot r}{\chi}  \le \frac{\gamma \chi \sum_{r=1}^{\lfloor\gamma \chi\rfloor}p_r}{\chi} \le \gamma.
\end{align*}
Hence, with probability at least $1-\gamma$, the \ooeaD proceeds from $x$ to a search point in Hamming distance at least $\gamma\chi -1$ from~$\vec 1$. 

For claim \textit{(a)} we set $\gamma \coloneqq 1/2$ and obtain a search point in Hamming distance at least $\chi/2 -1 \ge \chi/4$. Once this search point is reached, the algorithm needs $\Omega(n\ln^+(\chi/4))$ steps. 
Since this case happens with probability at least $1-\gamma = 1/2$, we obtain 
\[
\E[T^{x}_{\calD}(f_x)] \ge \tfrac{1}{2}\cdot \Omega(n\ln^+(\chi/4)) = \Omega(n\ln^+\chi).
\]

For claim \textit{(b)}, we choose $\gamma \coloneqq \chi^{-1/2}$. Then $1-\gamma = 1-o(1)$, so with high probability the \ooeaD proceeds from $x$ to a search point in distance at least $d\coloneqq \gamma \chi -1 = \chi^{1/2}-1 = \omega (1)$ from $\vec 1$. After this point, the algorithm still needs $\Omega(n\ln d)=\omega(n)$ steps in expectation and $\omega(n)$ steps with high probability.

%As for part (a), from this point onwards the algorithm needs to solve \onemax on $n-1$ bits. Thus the algorithm needs to traverse the interval from $d' \coloneqq \min\{d,n^{1/3}\}$ to the optimum. We show that w.h.p.\! this takes time at least $t_0$ for some $t_0 = \Omega(n\ln d') = \omega(n)$. It can be shown that the probability to find an improvement with $r$-bit flips for any $r\ge 2$ is $O((d'/n))^2 \le n^{-4/3}$. Hence, in time $t_0$ the expected number of such improvements is $O(t_0 n^{-4/3}) = o(1)$, and by Markov's inequality no $r$-bit flip finds an improvement for $r\ge 2$ with probability $1-o(1)$. Hence, we may pessimistically assume that the algorithm only uses single-bit flips, i.e., that it is random local search (RLS). By~\cite[Theorem~1]{witt2014fitness}, w.h.p.\! RLS needs time $\Omega(n\ln d')$ to find the optimum, which concludes the proof. 

Claim \textit{(c)} follows directly from the proof of Theorem~\ref{thm:upperbnd}, using the parameter $\alpha = 2$. There it was shown that the drift of the potential $g(x) = \sum_{i=1}^n g_i x_i$ towards zero (which corresponds to the optimum) is at least $\tfrac{p_1}{2n}\cdot g(x)$ by~\eqref{driftbnd}. By our definition of $g$, the minimal positive potential is 1. Since we start in distance one from the optimum, the initial potential is at most
\begin{align}
g_{\text{init}} \le \max\{g_i:i \in [n]\} = g_{n} \le \gamma_{n},
\end{align} 
where
\begin{align}
    \gamma_n = \bigg(1+ \frac{2 \chi^3}{(n-1)p_1^2} \bigg)^{n-1} \le \exp\bigg(\frac{2 \chi^3}{p_1^2} \bigg) = O(1)
\end{align}
by~\eqref{eq:def-gammai} and~\eqref{eq:def-gi}. Hence, the expected runtime is at most
\begin{align}
\E[T] \le \frac{\ln(g_{\text{init}})+1}{p_1/(2n)} = \frac{O(n)}{p_1} = O(n)
\end{align}
by the multiplicative drift theorem.\end{proof}

We did not aim for tightness in Proposition~\ref{prop:large-mu-general}, but rather want to demonstrate the different regimes. In particular, consider a distribution $\calD$ with $p_1 = \Theta(1)$ and with mean $\chi$. If $\chi = O(1)$, then $\E[T^{x}_{\calD}(f)] = O(n)$ by (c), but for $\chi = \omega(1)$ we have $\E[T^{x}_{\calD}(f)] = \omega(n)$ by (a). This shows that the value of $\chi$ is truly relevant for the runtime, at least if an algorithm starts in an adversarial point. Moreover, (b) shows that the high expectation in the case $\chi = \omega(1)$ is not just due to low-probability events, but that it comes from typical runs initialized in $x$.

The most interesting and common unbiased mutation operators, except for standard-bit mutation, are mutation operators where the number of bit flips has a \emph{heavy tail}. Usually a \emph{power law} is used, i.e, the probability to flip $k$ bits scales like $k^{-\beta}$ for some constant $\beta >1$. There are two different regimes for the parameter $\beta$. For $\beta >2 $, the expected number of bit flips satisfies $\chi = O(1)$. For $\beta \in (1,2)$, the expected number of bit flips is unbounded and large, $\chi = n^{\Omega(1)}$.\footnote{For $\beta =2$ the expected number of bit flips is also unbounded, but grows only as $\chi = O(\ln n)$. We will neglect this case here.} In either case, such power-law distributions satisfy $p_1 = \Theta(1)$. Notably, our main Theorem~\ref{thm:main-short} applies to power-law distributions with $\beta >2 $, but not to power-law distributions with $\beta \in (1,2)$. We believe that this reflects a real difference between those two cases. As indication, note that in the situation of Proposition~\ref{prop:large-mu-general}, we have $\E[T^{x}_{\calD}(f)] = O(n)$ for $\beta >2$, but $\E[T^{x}_{\calD}(f)] = \Omega(n\ln n)$ for $\beta \in (1,2)$. This does not rule out that Theorem~\ref{thm:main-short} still might be true for $\beta \in (1,2)$ since the starting point of the \ooea is drawn uniformly at random, but it suggests that trajectories of the algorithm can be substantially different.

\subsection{On $p_{n-1}$}
\label{sec:lower_pn1}

The following lemma shows that the term $p_{n-1}$ in Theorem~\ref{thm:lower-bound} is really necessary. In particular, there are (artificial) functions on which a unary unbiased $(1+1)$ algorithm with $p_1 = 0$ can be as efficient as random local search (RLS) on \onemax. Around the optimum, the function is similar to the \textsc{OneMix} function from~\cite{poli2006emergent} which oscillates between \onemax and \textsc{ZeroMax} depending whether the number of one-bits is even or odd. However, other than \textsc{OneMix} we extend this oscillation to the whole search space. If the \ooea uses only $(n-1)$-bit flips, this makes the algorithm behave exactly like RLS on \onemax,\footnote{For even $n$. For odd $n$, one would need to flip the parity condition in the half of the search space that does not contain the optimum.} in the strongest possible sense: they can be stochastically coupled. 

\begin{lemma}
\label{lem:pn-1}
Let $n$ be even. Let RLS be the \ooeaD with $p_{1} = 1$ and $p_i =0$ for $i\neq 1$, and let $\A$ be the \ooeaD with $p_{n-1} = 1$ and $p_i =0$ for $i\neq n-1$. Let $f:\{0,1\}^n \to \R$ be defined via
\begin{align*}
    f(x) \coloneqq \begin{cases} 
    \OM(x), & \text{ if } \OM(x) \text{ is even,} \\
    n-\OM(x), & \text{ if } \OM(x) \text{ is odd}.
    \end{cases}
\end{align*}
Let $T^{\text{RLS}}(\OM)$ be the runtime of RLS on \onemax, and let $T^{\A}(f)$ be the runtime of $\A$ on $f$. Then $T^{\text{RLS}}(\OM)$ and $T^{\A}(f)$ follow the same distribution, i.e., for all $T \in \N$,
\begin{align}\label{eq:lem:pn-1}
    \Pr[T^{\text{RLS}}(\OM) = T] = \Pr[T^{\A}(f) = T].
\end{align}
In particular, $T^{\A}(f) = (1\pm o(1))n\ln n$ in expectation and with high probability.
\end{lemma}

\begin{proof}
Let $X_t^{\text{RLS}}$ and $X_t^{\A}$ be the fitness of RLS and $\A$ after $t$ iterations, respectively. We will show by induction over $t$ that those two random variables follow the same distribution.

Note that $f$ is obtained from \onemax by swapping the fitness levels $k$ and $n-k$ if $k$ is odd. In particular, the number of search points of fitness $k$ does not change. Since the initial search point is chosen uniformly at random, therefore $X_0^{\text{RLS}}$ and $X_0^{\A}$ follow the same distribution. 

Now let $t\ge 0$ and assume that $X_t^{\text{RLS}}$ and $X_t^{\A}$ follow the same distribution. 
%Then we can couple $X_t^{\text{RLS}}$ and $X_t^{\A}$ such that $X_t^{\text{RLS}} = X_t^{\A}$. 
We claim that for every $k,k' \in[0..n]$, we have \begin{align}\label{eq:pn-1}
\Pr[X_{t+1}^{\text{RLS}} = k' \mid X_t^{\text{RLS}} = k] = \Pr[X_{t+1}^{\A} = k' \mid X_t^{\A} = k].
\end{align}
Note that this implies that $X_{t+1}^{\text{RLS}}$ and $X_{t+1}^{\A}$ follow the same distribution since then
\begin{align}
\begin{split}
    \Pr[X_{t+1}^{\text{RLS}} = k'] & = \sum_{k\in[0,n]} \Pr[X_t^{\text{RLS}} = k] \cdot \Pr[X_{t+1}^{\text{RLS}} = k' \mid X_t^{\text{RLS}} = k] \\
    & = \sum_{k\in[0,n]} \Pr[X_t^{\A} = k] \cdot \Pr[X_{t+1}^{\A} = k' \mid X_t^{\A} = k] \\
    & = \Pr[X_{t+1}^{\A} = k'].
\end{split}
\end{align}
Moreover, since this implies that $X_{t}^{\text{RLS}}$ and $X_{t}^{\A}$ follow the same distribution for all $t$, by
\begin{align}
    \Pr[T^{\text{RLS}}(\OM) > T] & = \Pr[X_{T}^{\text{RLS}} < n] 
     = \Pr[X_{T}^{\A} < n] = \Pr[T^{\A}(f) > T],
\end{align}
it also implies the lemma. So it remains to show \eqref{eq:pn-1}. 

For RLS it is obvious that the left hand side of \eqref{eq:pn-1} is zero for all $k' \in [0..n] \setminus\{k,k+1\}$. Let us assume that $\A$ is in a search point $x$ such that $X_t^{\A} = k$. The algorithm $\A$ creates offspring $y$ by randomly flipping $n-1$ positions of $x$. This can be equivalently expressed by first flipping all $n$ positions, and then flipping back a uniformly random position. Flipping all $n$ positions yields the antipodal search point $x'$ with $\OM(x')= n-\OM(x)$. Since $y$ is obtained from $x'$ by flipping exactly one bit, it satisfies $\OM(y) = n-\OM(x) \pm 1$. Since $n$ is even, this implies that $\OM(x')$ and $\OM(x)$ are either both odd or both even. In either case, $f(x') = n-f(x)$ and thus $f(y) = f(x) \pm 1$ by definition of $f$. Since $\A$ is elitist, it will reject any offspring of fitness $f(x)-1$, so it accepts $y$ if and only if $f(y) = f(x)+1$. In particular, this means that the right hand side of \eqref{eq:pn-1} is zero for all $k' \in [0..n] \setminus\{k,k+1\}$, as required. 

For the remaining values $k' \in\{ k,k+1\}$, it suffices to show equality for one of them, since the left and right hand side of \eqref{eq:pn-1} both sum up to one if summed over all $k'$. If $\OM(x)$ is even, then the offspring is fitter if and only if the bit that is \emph{not} flipped is a zero-bit, which happens with probability $(n-\OM(x))/n = (n-k)/n$. If $\OM(x)$ is odd, then the offspring is fitter if and only if the bit that is not flipped is a one-bit, which happens with probability $\OM(x)/n = (n-k)/n$. So in either case, $\Pr[X_{t+1}^{\A} = k+1 \mid X_t^{\A} = k] = (n-k)/n$, which is the same as the probability for RLS. This concludes the proof of \eqref{eq:pn-1} and of the lemma.
\end{proof}

The following theorem strengthens  Theorem~\ref{thm:lower-bound} for the \ooeaD on linear functions. It says that in this case, $p_{n-1}$ does not help to improve the asymptotic expected runtime.
\begin{theorem}\label{thm:linear-lower}
Consider the \ooeaD with distribution $\calD = (p_0,p_1,\ldots,p_n)$ such that $p_1 = n^{-o(1)}$. The expected runtime on any linear function on $\{0,1\}^n$ is at least 
     \begin{align}
         (1 - o(1)) \frac{1}{p_1} n \ln n.
     \end{align}    
\end{theorem}

\begin{proof}
Recall that the weights $w_i$ of $f$ are positive and sorted. We may assume that the smallest weight is $w_1 = 1$, since we can multiply all weights with the same constant factor without changing the fitness landscape. Moreover, for the proof we will work with minimization instead of maximization, which is equivalent. Finally, if $w_n > \sum_{i=1}^{n-1} w_i + 1$ then replacing $w_n$ by $\sum_{i=1}^{n-1} w_i + 1$ does not change the fitness landscape since in either case all search points $x$ with $x_n=1$ have higher objective than all search points with $x_n =0$. Hence we may assume $w_n \le \sum_{i=1}^{n-1} w_i + 1$. Writing $W \coloneqq \sum_{i=1}^n w_i$ for the total weight, this implies $2w_n \le W+1 < \tfrac32 W$, and thus $w_n < \tfrac34 W$. 

Fix some $t$, and let $q_{i,t} \coloneqq \Pr[x^{(t)}_i =1]$ be the probability that the $i$-th bit is a one-bit in generation~$t$. Then a classical result by J\"agersk\"upper~\cite{jagerskupper2008blend} says that $q_{1,t}\ge \ldots \ge q_{n,t}$. J\"agersk\"upper proved it for the \ooea with standard bit mutation, but the only ingredient in the proof was that for all $i,j\in [n]$, if we condition on the set of flips in $[n]\setminus \{i,j\}$ then positions $i$ and $j$ have the same probability of being flipped. This is true for all unbiased mutation operators, so J\"agersk\"upper's result holds for the \ooeaD as well. Moreover, J\"agersk\"upper's proof shows inductively for all times that for any substring $\tilde x$ on the positions $[n]\setminus \{i,j\}$, the combination ``$x_i=1$, $x_j=0$, $\tilde x$'' is more likely than the combination ``$x_i=0$, $x_j=1$, $\tilde x$'' if $i<j$. Since both options have the same number of one-bits, and since other options (with $x_i=x_j$) contribute equally to $q_{i,t}$ and $q_{j,t}$, it was already observed in~\cite{lengler2015fixed} that $q_{i,t}\ge q_{j,t}$ still holds if we condition on the number of one-bits $\OM(x^{(t)})$ at time $t$. Moreover, the statement also still holds if we replace $t$ by the hitting time $T = T(d) = \min\{t \ge 0 \mid \OM(x^{t}) \le d\}$, so we have $q_{1,T}\ge \ldots \ge q_{n,T}$.

We choose $T=T(d)$ for $d= n/\ln n$. %(Recall that we minimize.) 
Then we have $\sum_{i\in[n]}q_{i,T} = \E[\OM(x^{T})] \le d$ by definition of $T$, and hence
\begin{align}
    \E[f(x^{(T)})] = \sum_{i\in [n]} w_i \cdot q_{i,T} & \le \frac{\big(\sum_{i\in [n]} w_i \big)\cdot \big(\sum_{i\in [n]} q_{i,T}\big)}{n} 
    \le \frac{Wd}{n} = \frac{W}{\ln n},
\end{align}
where the second step is Chebyshev's sum inequality, since $w_i$ and $q_{i,T}$ are sorted opposingly.

By Markov's inequality, at time $T$ we have w.h.p.\! $f(x^{(T)}) \le W/8$. In the following we will condition on this event. We claim that then after time $T$, any offspring obtained by an $(n-1)$-bit flip is rejected. To see this, consider any $x$ with $f(x) \le W/8$. Any offspring $y$ that is obtained from $x$ by an $(n-1)$-bit flip has objective $f(y) \ge W - W/8 - w_n$, because the antipodal point of $x$ has objective $W-f(x) \ge W-W/8$, and flipping back a bit can decrease the objective by at most $w_n < \tfrac34 W$. Hence, $f(y) \ge W-W/8 -w_n > W/8$. Therefore, the offspring $y$ has higher (worse) objective, and is rejected. Hence, once the algorithm reaches objective at most $W/8$, all offspring obtained from $(n-1)$-bit flips are rejected. In other words mutations of $n-1$ bits are idle steps. This means that after time $T$, the \ooeaD behaves as the $(1 + 1)$-\text{EA}$_{\calD'}$, where we define $\calD' = (p_0',p_1',\ldots,p_n')$ by
\begin{align}
\begin{split}
    p_i' \coloneqq 
    \begin{cases}
    0 & \text{ if } i= n-1,\\
    p_0 + p_{n-1} & \text{ if } i= 0, \\
    p_i & \text{ otherwise.}
    \end{cases}
\end{split}
\end{align}
Recall that $T$ is defined in terms of $\OM(x^{(T)})$, even though we minimize a general linear function. Since it is unlikely to change the \onemax value by more than $\ln^2 n$ in one step, at time $T$ w.h.p.\! we have $\OM(x^{(T)}) \ge d-\ln^2 n$, see~\cite[Lemma~13]{DDY2020}. By Theorem~\ref{thm:lower-bound} and its proof (Equation~\eqref{eq:thm:lower-bound-summary}), the $(1 + 1)$-\text{EA}$_{\calD'}$ needs in expectation at least $(1-o(1))\frac{1}{p_1'+p_{n-1}'}n\ln n = (1-o(1))\frac{1}{p_1}n\ln n$ steps to find the optimum from level $d-\ln^2 n$, and hence the \ooeaD needs the same time. Note that we proved the lower bound conditional on w.h.p.\! events, but this just adds another $(1-o(1))$ factor for the unconditional expectation.
\end{proof}

\subsection{No Stochastic Domination}\label{sec:domination}
Earlier work~\cite{sudholt2013,djw2010revisited,witt2013tight} used stochastic domination arguments (cf. \cite{doerr2019domination}) to prove lower bounds. In particular, Witt proved his lower bound by showing that \onemax is the easiest function for the \ooea with standard bit mutation of arbitrary mutation rate $p\le 1/2$~\cite{witt2013tight}. The key ingredient was Lemma~6.1 in~\cite{witt2013tight}, which considered two offspring $y$ and $y'$ that are created from parents $x$ and $x'$ respectively by standard bit mutation with mutation rate $p\le 1/2$. For minimization, if $\OM(x) \le \OM(x')$ then the lemma states $\Pr[\OM(y) \le k] \ge \Pr[\OM(y') \le k]$ for all $k\in[0,n]$. So it is easier to reach $\OM$-level at most $k$ when starting with a parent of smaller $\OM$-value. This lemma implies on the one hand that \onemax is the easiest function for standard bit mutation, but also that elitist selection is optimal in this situation: the \ooea with mutation rate $p\le 1/2$ is the fastest algorithm on \onemax among all unary algorithms using standard bit mutation with mutation rate $p\le 1/2$.

However, Witt's lemma does not hold for general unbiased mutation operators. In particular, being closer to the optimum does not mean that we have a higher chance of finding the optimum in the next step. Consider the elitist $(1+1)$ algorithm which flips one bit with probability $p_1=n^{-2}$ and two bits with probability $p_2=1-p_1= 1-n^{-2}$. The probability of finding the optimum from a search point in Hamming distance one from the optimum is $p_1/n=n^{-3}$, whereas the probability of finding the optimum from a search point in Hamming distance two from the optimum is $p_2/\binom{n}{2} = \Theta(n^{-2})$, which is much larger.

Even worse, let us consider the time $T_d$ to find the optimum on \onemax if we start in Hamming distance $d$. For $d=1$ we have $\E[T_1] = n/p_1 = \Theta(n^3)$. For $d=2$, the probability of making an improvement is $p_{\text{imp}} = p_1\cdot 2/n + p_2/\binom{n}{2} = \Theta(n^{-2})$. Hence, conditional on making an improvement, the algorithm improves by one with probability $(p_1\cdot 2/n)/p_{\text{imp}} = \Theta(n^{-1})$. Therefore, we need to wait in expectation $1/p_{\text{imp}} = \Theta(n^2)$ rounds for an improvement, and with probability $\Theta(n^{-1})$ we improve only by one and need to wait another $T_1$ rounds for reaching the optimum. Hence,
\begin{align}
\E[T_2] = \Theta(n^2) + \Theta(n^{-1})\cdot \E[T_1] = \Theta(n^2),
\end{align}
which is asymptotically smaller than $\E[T_1] = \Theta(n^3)$. So the expected time $\E[T_d]$ is not monotone in $d$, and can be asymptotically smaller if we start further away from the optimum.

Turning this example around, we can construct a situation where \onemax is not the easiest function. Consider an algorithm with $p_1 = n^{-3}$, $p_2 = n^{-1}$ and $p_3 = 1-p_1-p_2 = 1-o(1)$, starting in the string $x = (01\ldots 1)$ where all but the first bit are optimized. On \onemax, it needs to wait for a one-bit flip, which takes time $n/p_1 = \Theta(n^4)$. But if the fitness function is $f(x) \coloneqq 3x_1 + \sum_{i=2}^n x_i$, then the algorithm accepts any mutation flipping two or three bits if it involves $x_1$. Conditional on flipping $x_1$, a two-bit flip has only probability $O(n^{-1})$ since $p_3/p_2 = \Theta(n)$. In that case (an improving two-bit flip) the algorithm jumps to another neighbour of the optimum and needs to wait $n/p_1 = O(n^4)$ rounds for the right one-bit flip. This contributes $O(n^{-1} \cdot n^4) = O(n^3)$ to the expectation. However, in the more likely case of a three-bit flip, the algorithm jumps to a search point in distance two from the optimum. By a similar calculation as before, it now needs time $O(n^3)$ to find the optimum, so the expected runtime on $f$ is $O(n^3)$, which is asymptotically faster than on \onemax.

Finally, the same example can be used to show that a non-elitist $(1+1)$ algorithm may be faster than the \ooeaD if both use the same unbiased mutation operator. Hence, Witt's lemma and all its consequences fail for general unbiased mutation operators. This is similar to the situation for the compact genetic algorithm cGA, for which this form of domination does not hold either~\cite{doerr2021cgajump}.

\section{Conclusions}
\label{sec:conclusions}
We have extended Witt's result bounding the runtime of the \ooea on linear functions to arbitrary elitist (1+1) unary unbiased EAs and we have discussed various ways in which the requirements made in Corollary~\ref{cor:upper-bound} and Theorem~\ref{thm:lower-bound} are tight. In particular, we have seen that for $p_1 = n^{-\Omega(1)}$, the expected runtime can be smaller than $\tfrac{1}{p_1} n \ln n$ by a constant factor. 
When interpreted in the light of black-box complexity, our results can be seen as extensions of~\cite{DDY2020} to linear functions. However, we have focused in this work on \textit{static} mutation operators. An extension of our result to \textit{dynamic} parameter settings would hence be a natural continuation of our work.  

Another direction in which we aim to extend our results are \textit{combinatorial} optimization problems where we suspect to see a tangible advantage of unusual unary mutation operators. For example, the optimal mutation operator for the minimum spanning tree problem (MST) is likely to satisfy $p_1>0$ and $p_2>0$. Similarly, there are functions like \leadingones where the optimal number of flipped bits depends on the phase of the algorithm, and none of the phases is asymptotically negligible for the runtime. In such cases, it may be interesting to see what the optimal distribution is.

Similarly, we also expect advantages of the \ooeaD over standard $(1 + 1)$-EAs when optimizing for average performance for problem collections with instances having different landscapes.

\backmatter

\bmhead{Acknowledgments}
We thank the anonymous reviewers for many helpful comments and suggestions. In particular, one reviewer pointed out a simplification in the proofs of Lemma~\ref{lem:Bndr_bound} and Lemma~\ref{lem:h_d_p1_On2}.
Our work is financially supported by ANR-22-ERCS-0003-01 project VARIATION. Duri Andrea Janett was supported by the Swiss-European Mobility Programme.

\section*{Declarations}
The authors have no competing interests to declare that are relevant to the content of this article.

% \bibliography{references}

\end{document}